\documentclass{article}

\usepackage{arxiv}

\usepackage[utf8]{inputenc} 
\usepackage[T1]{fontenc}    
\usepackage{hyperref}       
\usepackage{url}            
\usepackage{booktabs}       
\usepackage{amsfonts}       
\usepackage{nicefrac}       
\usepackage{microtype}      
\usepackage{graphicx}
\graphicspath{ {./images/} }

\usepackage{amsmath,amsthm,bbm}
\usepackage{amsfonts,amssymb}

\usepackage[noend]{algpseudocode}
\usepackage{algorithmicx,algorithm}
\usepackage{multirow}
\usepackage{graphicx}
\usepackage{epstopdf}
\usepackage{subfigure}
\usepackage{parskip}
\usepackage{bm}
\usepackage{color}

\usepackage{appendix}

\usepackage{amsfonts,amssymb}
\usepackage{amsmath,amsxtra,amssymb,latexsym, amscd,amsthm}
\usepackage{indentfirst}
\usepackage{fancyhdr}

\usepackage[mathscr]{eucal}
\usepackage{mathrsfs}
\usepackage{makeidx}
\usepackage{mathabx}
\usepackage{color}

\usepackage{bm}

\newtheorem{thm}{Theorem}

\newtheorem{lem}{Lemma}
\newtheorem{rem}{Remark}
\newtheorem{prop}{Proposition}

\newtheorem*{definition}{Form of approximation by SignReLU nets}

\newcommand{\thmref}[1]{Theorem~\ref{#1}}

\newcommand{\lemref}[1]{Lemma~\ref{#1}}

\newcommand{\propref}[1]{Proposition~\ref{#1}}

\def\RR{\mathbb{R}}
\def\EE{\mathbb{E}}
\def\E{\mathcal{E}}

\def\s{\sigma}
\def\NN{\mathbb{N}}
\def\RR{\mathbb{R}}
\def\ZZ{\mathbb{Z}}
\def\f{\frac}

\def\b{\bm}
\def\p{\partial}

\def\t{\theta}
\def\bk{\b{k}}
\def\bt{\b{\theta}}

\def\L{\mathcal{L}}
\def\W{\mathcal{W}}
\def\N{\mathcal{N}}
\def\A{\mathcal{A}}

\DeclareMathOperator*{\esssup}{ess\,sup}

\title{signrelu neural network and its approximation ability}

\author{
	{\bf Jianfei Li}\thanks{Department of Mathematics, City University of Hong Kong (\texttt{jianfeili2-c@my.cityu.edu.hk})}
	\and
	{\bf Han Feng}\thanks{Department of Mathematics, City University of Hong Kong (\texttt{hanfeng@cityu.edu.hk})}
	\and
	{\bf Ding-Xuan Zhou}\thanks{School of Mathematics and Statistics, The University of Sydney (\texttt{dingxuan.zhou@sydney.edu.au})}
}

\date{}

\begin{document}
\maketitle

\begin{abstract}
	Deep neural networks (DNNs) have garnered significant attention in various fields of science and technology in recent years. 
Activation functions define how neurons in DNNs process incoming signals for them. They are essential for learning non-linear transformations and for performing diverse computations among successive neuron layers.
	In the last few years, researchers have investigated the approximation ability of DNNs to explain their power and success. In this paper, we explore the approximation ability of DNNs using a different activation function, called SignReLU. Our theoretical results demonstrate that SignReLU networks outperform rational and ReLU networks in terms of approximation performance. Numerical experiments are conducted comparing SignReLU with the existing activations such as ReLU, Leaky ReLU, and ELU, which illustrate the competitive practical performance of SignReLU.
\end{abstract}






keywords: Deep neural networks, Activation function, Approximation power, SignReLU activation




\section{Introduction}


{\bf Deep learning} has become a critical method in developing AI for handling complicated real-world tasks that appeared in human societies.
Wide applications of deep learning including those in image processing \cite{liu2021review, jiao2019survey} and speech recognition \cite{nassif2019speech, santhanavijayan2021semantic} have
received great successes in recent years. Theoretical explanations for the success of deep learning have been recently studied from the point of view of approximation theory.

A fully connected deep neural network (DNN) of input $\b x=(x_1,x_2,\ldots, x_d)^\top \in \RR^d$ with depth $\L-1$ is defined as
\begin{align}\label{defnn}
    \Phi = \mathcal{A}_{\L} \circ \sigma \circ \mathcal{A}_{\L-1} \circ \sigma \circ \cdots \circ  \sigma \circ \mathcal{A}_2 \circ \sigma \circ \mathcal{A}_1 ,
\end{align}
where  $\mathcal{A}_i(\b x) := \b A_i \b x + \b b_i$ are affine transforms with weight matrices $\b A_i \in \RR^{d_{i} \times d_{i-1} }$ and bias vectors $\b b_i \in \RR^{d_{i}}$ and $\s:\RR\to \RR$ is an activation function acting on each element of input vectors. We call $\s \circ \mathcal{A}_i$ in $\Phi$ the $i$-th layer (hidden layer) with width $d_i$ and $\mathcal{A}_{\L}$ the output layer. We say a neural network has width $\W$ if the maximum width $\max_{1\leq i \leq L} \{d_i\}$ is no more than $\W$. We call the number of nonzero elements  of all $\b A_i$ and $\b b_i$ in the neural network the number of weights (size) of the neural network $\Phi$, denoted by $\N$. It is apparent that activations are a key ingredient of the nonlinearity of deep neural networks. 


 So far, many activation functions have been developed for various tasks. Neural networks in the early days utilized sigmoid-like activations, which however suffer from gradient vanishing problems \cite{bengio1994learning}, especially as neural networks become deep. To overcome this phenomenon, a new activation was found and popularized \cite{NairHinton2010, glorot2011deep}, called the rectified linear unit (ReLU), and from then on, artificial intelligence began to flourish with the development of deep learning techniques.
Despite the remarkable success of ReLU, it has a significant limitation that in each neuron ReLU squishes negative inputs to zeros which may cause information loss at the feed-forward stage, and some dead neurons appear during training .

To resolve such issues, some modifications of ReLU were proposed and gained widespread popularity, including LeakyReLU \cite{maas2013rectifier}, Parametric Linear Unit  (PReLU) \cite{he2015delving}, Softplus function \cite{dugas2000incorporating}, Exponential Linear Unit (ELU) \cite{clevert2015fast} and Scaled Exponential Linear Unit (SELU) \cite{klambauer2017self}. A lot of modified activations have the same expression as $\text{ReLU}(x) = \max\{x,0\}$ on $(0,\infty)$. For example, ELU and LeakyReLU are defined as
\begin{equation}\label{elu}
\text{ELU}(x;\alpha):=\left\{
\begin{array}{rcl}
&x,   &{\text{if } x \in [0,\infty),}\\
&\alpha(e^x-1),   &{\text{if } x \in (-\infty,0),}
\end{array} \right.
\quad
\text{LeakyReLU}(x;\alpha):=\left\{
\begin{array}{rcl}
&x,   &{\text{if } x \in [0,\infty),}\\
&\alpha x,   &{\text{if } x \in (-\infty,0).}
\end{array} \right.
\end{equation}
These modified activation functions show promising improvements on several tasks compared with ReLU \cite{nwankpa2018activation}.
Except for these monotonic activation functions, nonmonotonic activations, for example, Swish \cite{ramachandran2017searching}, Mish \cite{misra2019mish} and Logish \cite{zhu2021logish} were proposed and shown great performances in various tasks.

PDELU \cite{cheng2020parametric} is one of the most recently proposed activation functions, defined as,
\begin{equation}\label{SignReLU}
\text{PDELU}(x;\alpha,t)=\left\{
\begin{array}{rcl}
&x,   &{\text{if } x \in (0,\infty),}\\
&\alpha \left \{ \big[1+ (1-t)x \big]^{\f{1}{1-t}}-1 \right \},   &{\text{if } x \in (-\infty,0],}
\end{array} \right.
\end{equation}
with $\alpha$ and $ t$ controlling the slope and the degree of deformation, respectively. It is verified to have many desired properties, for example, speeding up the training process and possessing geometric flexibility. Its effectiveness is observed on many datasets and well-known neural network architectures (including NIN, ResNet, DenseNet) \cite{cheng2020parametric, nanni2021comparison, dubey2022activation}. SignReLU function \cite{lin2018research} is inspired by the softsign function, defined as:
\begin{equation}\label{sr}
\text{SignReLU}(x;\alpha)=\left\{
\begin{array}{rcl}
&x,   &{\text{if } x \in (0,\infty),}\\
&\alpha \f{x}{1+|x|},   &{\text{if } x \in (-\infty,0].}
\end{array} \right.
\end{equation}
It is easy to verify that SignReLU \eqref{sr} corresponds to the special case of PDELU \eqref{SignReLU} when $t = 2$. In some experiments, SignReLU improves the convergence rate and alleviates the gradient vanishing problem in image classification tasks \cite{lin2018research}.


	
    There exist lots of different explanations for choosing a proper activation function.
    In learning theory, the learning ability of a neural network is closely related to approximation error \cite{lu2021deep, shen2021deep}. In deep learning, learning tasks aim to find a proper model $\Phi(\b x; \b w)$ parameterized by $\b w \in \RR^\N$ which approximates target function $f(\b x)$ well. The performance of a learned model $\Phi(\b x; \hat{\b w})$ ($\hat{\b w}$ is learned through a learning algorithm) over a sampled dataset $Z = \left \{(\b x_i, \b f(\b x_i)) \right \}_{i=1}^{N}$ can be measured by a loss function $L(\b x,y)$. The generalization error and optimization error of model $\Phi$ with parameter $\b w$ is characterized by $\E( \b w):=\EE_{\b x}\left[ L\big(\Phi(\b x; \b w), f(\b x) \big)\right]$ and $\E_Z( \b w):= \f{1}{N}\sum_{i=1}^{N} L\big(\Phi(\b x_i; \b w), f(\b x_i) \big)$, respectively. Then, the performance of $\Phi(\b x; \hat{\b w})$ in learning theory can be upper-bounded by
	\begin{align}\label{gener_error}
	\E(\hat{\b w})
	\leq \E(\b w^*) + [ \E_Z (\hat{\b w}) - \E_Z ( \b w_N ) ] + \Big \{[\E(\hat{\b w}) - \E_Z(\hat{\b w})] + [\E_Z(\b w^*) - \E(\b w^*)]\Big \},
	\end{align}
	where $\b w^* := \mathop{\arg\min}_{\b w \in \RR^W} \E (\b w)$ is the parameter with the best generalization error and $\b w_N = \mathop{\arg\min}_{\b w \in \RR^W} \E_{ Z}( \b w)$ is the parameter with the smallest optimization error.
	The first, second, and third term in \eqref{gener_error} are called approximation error, optimization error, and generalization error, respectively. Obviously, controling approximation error $\E(\b w^*) $ is of great importance to control $\E(\hat{\b w}) $. See more details in \cite{lu2021deep, shen2021deep}.
	
    In practical experiments, to find a model suitable for learning tasks, one also needs to balance performance and efficiency. A lot of work for classification, object detection, and video coding attempts to compress the model while keeping the performance in order to improve inference time and lower memory usage \cite{qin2020u2,liu2020deep,howard2017mobilenets,sandler2018mobilenetv2}. The inference time and memory usage of a deep neural network depend heavily on the expression of the activation functions and the total number of parameters of deep neural networks.
    This kind of problem can be stated as characterizing the approximation error with the total number of parameters of deep neural networks. Therefore, it is worth focusing more on the approximation properties of valuable activation functions.

\subsection{Related work}
Theoretical studies on the approximation ability of deep neural networks with various activation functions have been developed in a large literature {\cite{zhou2018deep, abdeljawad2022approximations, guhring2020error, chui2019deep}}. Although neural networks are of great success in practical applications, existing theoretical results mainly focused on sigmoid type and ReLU activations.
When the activation function $\s$ is a $C^\infty$ sigmoid type function, which means $\lim_{x\to \infty}\s(x)=1$ and $\lim_{x\to -\infty}\s(x)=0$, the approximation rates were given by Barron \cite{Barron} for functions $f\in L_2({\RR}^d)$ whose Fourier transforms $\hat{f}$
satisfy a decay condition $\int_{\RR^d} |w| |\hat{f}(w)| d w<\infty$.
Another remarkable result (e.g. Mhaskar \cite{mhaskar1993approximation}) based on localized Taylor expansions asserts rates of approximation for functions from Sobolev spaces.
These results were developed by the localized Taylor expansion approach under the assumption that $\sigma$ satisfies $\sigma^{(k)}(\mu) \neq 0$ for some $\mu \in \RR$ and every $k \in \ZZ_{+}$. This condition is not satisfied by ReLU-type activations. 
Until recent years, approximation properties were established in \cite{Klusowski2018, MaoZhou} for shallow ReLU nets and in \cite{Yarosky, Grohs, Petersen} for deep ReLU nets for target functions from Sobolev spaces
and in \cite{shen2019deep} for general continuous functions.

Besides, related analysis has been also investigated for some other kinds of activation functions.
In \cite{shen2021deep}, the authors proposed Elementary Universal Activation Function (EUAF). They proved that all functions represented by a EUAF fully connected neural network (FNN) with a fixed structure are dense in the space of continuous functions, which was proved impossible with ReLU FNNs \cite{shen2019deep}. Unfortunately, EUAF is partly a nonsmooth periodic function which makes it not applicable in practice.
Another investigation about rational activations has been developed in \cite{boulle2020rational}. It was shown that rational neural networks learn smooth functions more efficiently than ReLU neural networks.
Meanwhile, some numerical experiments illustrated their potential power for solving PDEs and GAN. Due to the essential impact of activations on neural networks' performance, investigating the properties of activation functions is still an essential topic in deep learning research today.


In this paper, we study the approximation ability of SignReLU \eqref{sr}, and for simplicity, we fix $\alpha=1$. Precisely, we investigate the activation $\rho: \RR \rightarrow \RR$, given by
\begin{equation}\label{rho}
\rho(x):=\text{SignReLU}(x;1)=\left\{
\begin{array}{rcl}
&x,   &{\text{if } x \in [0,\infty),}\\
&\f{x}{1-x},   &{\text{if } x \in (-\infty,0).}
\end{array} \right.
\end{equation}

It is easy to see that $\rho$ is monotonically increasing on $\RR$ and $\rho(x) \rightarrow -1$ as $x \rightarrow -\infty$. Thus, it can be categorized as a ReLU-type activation function. Particularly, it enjoys a similar shape to ELU. The design on the negative part is the same as that of EUAF \cite{shen2021deep}. It is important since it allows SignReLU to represent the division gate as well as the product gate, which cannot be produced by ReLU neural networks with finite parameters.

\begin{figure}[http]
	\centering
	\subfigure[Activation functions]{\includegraphics[width=7cm]{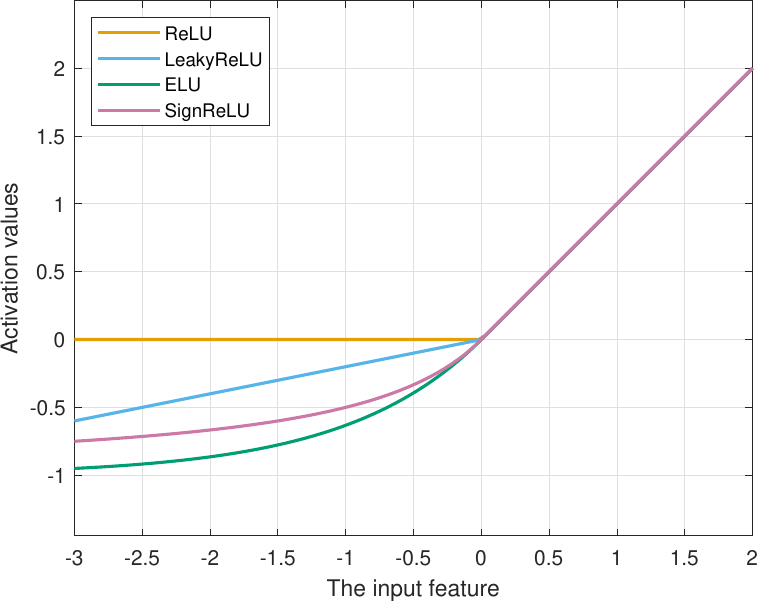}}
	\subfigure[Corresponding Gradients]{\includegraphics[width=7cm]{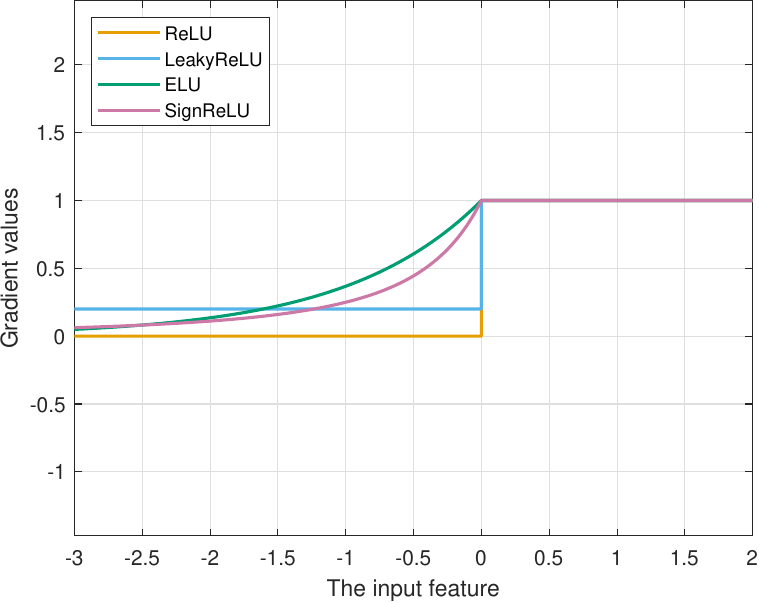}}
	\caption{Visualization of activation functions (ReLU, LeakyReLU [$\alpha =  0.2$], ELU [$\alpha = 1$] and SignReLU [$\alpha = 1$]) and their gradients.}
	\label{acts}
\end{figure}

Our main results in the following form quantify the structure (i.e., depth, layer, size) of SignReLU neural networks that guarantee certain approximation accuracy for a given target function. In this paper, all proofs are given in Appendix.

\begin{definition} Let $d \in \NN$. Let $\Omega \subset \RR^d$ and $f \in F$ be a function from function class $F$ defined on $\Omega$. For any $\varepsilon > 0$, there exists a function $\Phi$ implemented by a SignReLU neural network with (depth) $\L_{\Phi} = \L_{\Phi}(\varepsilon,f)$, (width) $\W_{\Phi} = \W_{\Phi}(\varepsilon,f)$ and (number of weights) $\N_{\Phi} = \N_{\Phi}(\varepsilon,f)$ such that
	$$|f(\b x)-\Phi(\b x)|\leq \varepsilon, \quad \forall\ \b x\in\Omega.$$
\end{definition}

When there exists a SignReLU neural network that equals $f$ on $\Omega$, then $\L_{\Phi} = \L_{\Phi}(f)$, $\W_{\Phi} = \W_{\Phi}(f)$ and $\N_{\Phi} = \N(f)$ only depends on the properties of $f$.

{\bf Contributions.} First, we characterize some basic properties that SignReLU neural networks possess, for example, realizing product gate and division gate, realizing rational functions, and approximating ReLU and exponential functions effectively. Then we show that given tolerance $\varepsilon>0$, SignReLU nets can approximate  Sobolev functions $W_p^r$ with non-zero parameters increasing at a rate of $\varepsilon^{-d/r}$. Moreover, the optimal approximation error for piecewise smooth functions is obtained. We also investigate when SignReLU neural networks overcome the curse of dimensionality, which is a crucial topic in machine learning. Our results show that when $\varepsilon$-approximating Korobov functions or BV functions, the dominant term in the size of neural networks is $\varepsilon^{-1/r}$, which is independent of the input dimension $d$. Finally, some numerical experiments are conducted on classification and image denoising tasks, illustrating competitive performances compared with the existing activations--ReLU, Leaky ReLU, and ELU. For the sake of convenience, we shall denote ReLU by $\sigma(x) = \max\{x,0\}$ in the rest of this paper.

\subsection{Notations}
Let us first clarify the basic notations to be used throughout the paper. Let $\RR$ denote all the real numbers, $\NN$ denote natural numbers, $\NN_{+}$ denote nonzero positive integers and $\ZZ$ denote the set of integers. Usually, we use $d$ or $d_i$ for some $i \in \NN$ to denote the dimension of a vector. We use boldface lowercase letters to denote a $d$-dimensional vector, for example, $\b x \in \RR^d$ with each element written as $x_i$, which means that $\b x = (x_1,x_2,\dots,x_d)^\top$. For some $\b x \in \RR^d$ and $\b n \in \ZZ^d$, we define $\b x^{\b n}$ as the operation $\prod_{i=1}^d x_i^{n_i}$. We denote the measure of a set $\Omega$ in $\RR^d$ as $\left|\Omega \right|$.

\section{Theoretical results for classifier functions}

Piecewise smooth functions are closely related to classification problems \cite{Petersen}. In image classification tasks, one needs to approximate a label for any given image $\b x$. If labels are from an integer set $\{1,2,\dots, K \}$, then the problem is to find the best piecewise constant function $f(\b x) = \sum_{i=1}^K i \chi_{\Omega_i}(\b x)$, which output $i$ if $\b x \in \Omega_i$ and zero otherwise. Function $\chi_{\Omega}$ is defined as $\chi_\Omega(\b x)=1$ if $\b x\in \Omega$, and zero otherwise. However, it is not easy to produce integer labels with neural networks. The most commonly used setting is to learn a distribution. If two images are close to each other, then it is reasonable to expect the probabilities of the labels to be close to each other. It means that one can assume that distributions are some Sobolev functions. Then the classification problems reduce to learn piecewise smooth functions $\sum_{i=1}^{L}f_i(\b x)\chi_{\Omega_i}(\b x)$, where $f_i$ are Sobolev functions and $\Omega_i \subset \RR^d$ are disjoint. Besides, smoothness also helps neural networks yield robust results since it is unlikely to be sensitive with respect to noisy images. 



\subsection{Basic properties of SignReLU neural networks}
This subsection is devoted to some useful and basic properties of SignReLU neural networks, which will be applied frequently in our approximation theory.
The first lemma shows that SignReLU neural networks are able to produce product/division gates.

\begin{lem}\label{times}
		Let $0<a<M$ be constants.
		\begin{enumerate}
			\item[\rm(i)] There exists a function $\Phi$ realized by a SignReLU neural network with $\L_{\Phi} = 4$, $\W_{\Phi} \leq 9$ and $\N_{\Phi} \leq 63$ such that
			\[\Phi(x,y) = xy,\quad \forall  x,y \in [-M,M].\]
			\item[\rm(ii)] There exists a function $\Phi$ realized by a SignReLU neural network with $\L_{\Phi} = 6$, $\W_{\Phi} \leq 9$ and $\N_{\Phi} \leq 71$ such that
			\[\Phi(x,y) = \f y x, \quad \forall  x \in [a,M], \ y \in [-M,M].\]
		\end{enumerate}
\end{lem}

Based on the above results, we are able to construct SignReLU networks to implement polynomials and rational functions.

\begin{lem}\label{ra}
		Let $M>0$ and $m,n \in \NN_+$. Let $p, q$ be polynomials with degrees at most $n$ and $m$, respectively. If  $q$ has no roots on $[-M,M]$, then there exists a function $\Phi$ realized by a SignReLU neural network with $\L_{\Phi} = O(\max\{n,m\})$, $\W_{\Phi} = O(1)$ and $\N_{\Phi} = O(\max\{n,m\})$ such that
		$$\Phi(x) = \f{p(x)}{q(x)},\quad \forall x \in [-M,M].$$
\end{lem}

\lemref{ra} shows the superiority of SignReLU for realizing polynomials and rational functions, compared with ReLU. Results in \cite{liang2016deep} show that at least $O(\ln(\varepsilon^{-1})) $ parameters are needed when using a ReLU neural network to $\varepsilon$-approximate $x^2$ on $[-1,1]$ .

The following lemma shows how ReLU can be approximated by SignReLU neural networks.

\begin{lem}\label{relu}
		Let $m,n \in \NN_+$.
		\begin{enumerate}
			\item[\rm(i)] For any $ \varepsilon>0$, there exists a function $\Phi$ realized by a SignReLU neural network  with $\L_{\Phi} = 1$ and $\W_{\Phi} = 1$ such that
			\[|\sigma(x)-\Phi(x)|\leq \varepsilon  \quad \forall  x \in \RR .\]
			
			\item[\rm(ii)] There exists a function $\Phi$ realized by a SignReLU neural network with $\L_{\Phi} = m$ and $\W_{\Phi} = 1$ such that
			\begin{align*}
			|\sigma(x)-\Phi(x)| &\leq  m^{-1}, \quad \forall  x \in \RR .
			\end{align*}
			
			\item[\rm(iii)] There exists a function $\Phi$ realized by a SignReLU neural network with $\L_{\Phi} = m$ and $\W_{\Phi} = 1$ such that
			\begin{align*}
			|\sigma(x)-\Phi(x)| &\leq  (mn)^{-1}, \quad \forall  x \in \RR.
			\end{align*}
		\end{enumerate}
\end{lem}

Lemma 1 in \cite{boulle2020rational} says that to approximate ReLU with tolerance $\varepsilon$, the size of rational neural networks needed is of order at least $  C\ln \ln \left(\varepsilon^{-1}\right)$ for some $C>0$. In comparison, SignReLU can approximate ReLU uniformly on $\RR$  instead of $[-1,1]$ with fixed size (\lemref{relu} (i)).


The exponential function is widely utilized in approximation and regression problems with Gaussian reproducing kernel Hilbert space \cite{smale-zhou-2007, RKHS-Gauss}. In the following proposition, we particularly give the power of SignReLU networks for approximating exponential functions.

\begin{prop}\label{exp}
    Let $M>0$ and $d \in \NN_+$.
     	\begin{enumerate}
			
			\item[\rm(i)] For any $ \varepsilon>0$, there exists a function $\Phi$ realized by a SignReLU neural network  with $\L_{\Phi} = O\left(\ln(\varepsilon^{-1})\right)$,
	$\W_{\Phi} = O(1)$ and $\N_{\Phi} = O\left(\ln(\varepsilon^{-1})\right)$ such that
	\begin{align*}
	\Big|e^{-||\b x||_1}- \Phi(\b x) \Big| &\leq  \varepsilon, \quad \forall \b
	x \in \RR^d,
	\end{align*}
	where $\|\b x\|_1=\sum_{j=1}^d |x_j|$.

			\item[\rm(iii)] For any $ \varepsilon>0$, there exists a function $\Phi$ realized by a SignReLU neural network  with $\L_{\Phi} = O\left(\ln(\varepsilon^{-1})\right)$,
	$\W_{\Phi} = O(1)$ and $\N_{\Phi} = O\left(\ln(\varepsilon^{-1})\right)$ such that
	\begin{align*}
	\Big|e^{-||\b x||_2^2}- \Phi(\b x) \Big| &\leq  \varepsilon, \quad \forall \b
	x \in [-M, M]^d,
	\end{align*}
	where $\|\b x\|_2=\sqrt{\sum_{j=1}^d |x_j|^2}$.

		\end{enumerate}
\end{prop}


\begin{rem}\label{rem1}
Applying the same argument, for arbitrary real constants $a, b$ and non-zero $c$, a Gaussian function $f(x)=a e^{-(x-b)^2/c}$ can be approximated by an $L$-layer SignReLU network with accuracy $e^{-L}$ which is exponentially decreasing as well.
\end{rem}


Let us recall the approximation results of ReLU neural networks and rational neural networks and compare them with approximation properties of SignReLU neural networks. Telgarsky discussed approximation relationships between ReLU neural networks and rational functions \cite{telgarsky2017neural}. To approximate a rational function with accuracy $\varepsilon$ by a ReLU neural network, the size needed is of order $O\left((\ln(\varepsilon^{-1}))^{3}\right)$. \lemref{times} and \lemref{ra} show that using a SignReLU neural network instead, the size for approximating a given rational function is independent of $\varepsilon$, which only depends on the degree of the rational function. 

Rational neural networks, activated by rational functions, could be more powerful than ReLU neural networks due to the following results obtained in \cite{boulle2020rational},
\begin{align}\label{bou}
\begin{aligned}
\text{Rational}\left [ \ln\ln (\varepsilon^{-1}) \right] &\leq \text{ReLU} \leq  \text{Rational}\left [ \ln\ln (\varepsilon^{-1}) \right], \\
\text{ReLU}\left [ \ln (\varepsilon^{-1}) \right] &\leq \text{Rational} \leq  \text{ReLU}\left [ (\ln (\varepsilon^{-1}))^3 \right],
\end{aligned}
\end{align}
where we use the expression \emph{$\text{Rational} \leq \text{ReLU}\left[ \N(\varepsilon) \right]  $} to represent when approximating a rational neural network by a ReLU neural network within the tolerance of $\varepsilon$, the needed size of ReLU neural networks is at most $C\N(\varepsilon)$ for some constant $C$. Conversely, the expression \emph{ $\text{ReLU}\left[ \N(\varepsilon) \right] \leq \text{Rational}$} represents that any ReLU neural network with the size less than $C\N(\varepsilon)$ for some constant $C$ cannot approximate the given rational neural network within tolerance $\varepsilon$.

Based on \lemref{ra}, \lemref{relu}, we are able to obtain the following improved approximation results by using SignReLU neural networks
\begin{align}\label{our}
    \begin{aligned}
    \text{SignReLU}\left [ 1 \right] &\leq \text{ReLU} \leq  \text{SignReLU}\left [ 1 \right], \\
    \text{SignReLU}\left [ 1 \right] &\leq \text{Rational} \leq  \text{SignReLU}\left [ 1 \right].
\end{aligned}
\end{align}



The above comparison suggests that SignReLU could be more powerful than ReLU and rational activation functions.
A key difference between SignReLU and other mentioned activations is its strong ability to approximate ReLU (\lemref{relu}) and product/division gates. Other activations do not possess these properties simultaneously. 

\thmref{relunets} theorectically verifies the claimed properties in \eqref{our}.

\begin{thm}\label{relunets}
            Let $d\in \NN_+$. Denote $\Phi_R$ / $\Phi_{\sigma}$ a rational neural network / ReLU neural network with depth $\L_{\Phi_R}$ / $\L_{\Phi_{\sigma}}$, width $\W_{\Phi_R}$ / $\W_{\Phi_{\sigma}}$ and number of weights $\N_{\Phi_R}$ /$\N_{\Phi_{\sigma}}$, respectively.
		\begin{enumerate}
			\item[\rm(i)] There exists a function $\Phi$ realized by a SignReLU neural network with $\L_{\Phi} = O(\L_{\Phi_R}) $, $\W_{\Phi}=O(\W_{\Phi_R}) $ and $\N_{\Phi}=O(\N_{\Phi_R}) $ such that
			\[\Phi(\b x) = \Phi_R(\b x),\quad \forall \b x \in [-1,1]^d.\]
			\item[\rm(ii)] For any $\varepsilon>0$, there exists a function $\Phi$ realized by a SignReLU neural network with $\L_{\Phi} = \L_{\Phi_{\sigma}} $, $\W_{\Phi}=\W_{\Phi_{\sigma}} $ and $\N_{\Phi}=\N_{\Phi_{\sigma}} $ such that
			\[ \left| \Phi(\b x) - \Phi_\sigma(\b x) \right| \leq \varepsilon,\quad \forall \b x \in \RR^d.\]
		\end{enumerate}
\end{thm}

\subsection{Approximation of weighted Sobolev smooth functions}\label{Fnd}

For $\bk \in \NN^d$ and $\b x \in [-1,1]^d$, let
\[D^{\bk}f= \frac{\partial^{|\bk|_1} f}{\partial x_1^{k_1}\cdots \p x_d^{k_d}} \ \text{and} \ w(\b x)=2^d\prod_{j=1}^d (1-x_j^2)^{-1/2}.\]
We consider the weighted Sobolev space $W^r_{p}([-1,1]^d, w)$ with $r>0$ and $1\leq p<\infty$ defined by locally integrable
functions on $[-1,1]^d$ with  norm
\[\|f\|_{W^r_{p}}:= \sum_{|\b k|_1\leq r}\left[\int_{[-1,1]^d}|D^k f(\b x)|^p w(\b x)d\b x\right]^{1/p} <\infty, \]
and $W^{r}_{\infty}\left([0,1]^{d}\right)$ by functions $f\in C\left([0,1]^d\right)$ with norm
\[\|f\|_{W^r_\infty}:= \max _{\b{k}:|\b{k}|_1 \leq r} \esssup_{\b x \in[0,1]^{d}}\left|D^{\b{k}} f(\b{x})\right| < \infty,\]
where $|\b{k}|_1=k_{1}+\ldots+k_{d}$.

The error between target functions and neural networks is measured by
\begin{align*}
    \|f\|_{p,w,\Omega} &:= \left[ \int_{\Omega} |f(\b x)|^p w(\b x) d \b x \right]^{1/p} ,\\
    \|f\|_{\infty,w,\Omega} &:= \esssup_{\b x \in \Omega} \left|  |f(\b x)| \right|.
\end{align*}
For simplicity, we denote $\|\cdot\|_{p,w,\Omega}:= \|\cdot \|_{p,w}$ or $\|\cdot\|_{p,w,\Omega}:= \|\cdot \|_{p}$, if $w(\b x) \equiv 1$ or $\Omega$ is clear from the context. Notice that $\|\cdot\|_p$ is the classical $L_p$ norm.


The motivation for introducing weighted spaces is from a technical perspective, which allows us to apply approximation analysis by trigonometric polynomials, instead of localized Taylor expansions for developing approximation results. During changing variables, the weight function $w(\b x)$ will appear. Compared to techniques used in \cite{boulle2020rational, Yarosky}, our results extend $p=\infty$ to $p \in [1,\infty)$ and can be directly applied to all other activation functions that are able to realize polynomial/rational functions. Besides, the constant of the obtained size $O(\varepsilon^{-d/r})$ in our result is $CC_rd$, for some constants $C, C_r>0$ where the constant $C_r$ depends only on $r$. This is better than those that appeared in \cite{boulle2020rational, Yarosky}.

In the following result, we shall show the approximation ability of SignReLU neural networks to functions in the weighted Sobolev spaces. It shows that with SignReLU activation an improved rate can be achieved.
\begin{thm}\label{sobolev}
	Let $d, r\in \NN_+$ and $1\leq p < \infty$.

        \begin{enumerate}
            \item [\rm(i)] Let $f$ from the unit ball of $ W^r_p([-1,1]^d,w)$.

            For any $\varepsilon>0$, there exists a function $\Phi$ realized by a SignReLU neural network with $\L_{\Phi}=O(\varepsilon^{-\f{1}{r}})$, $\W_{\Phi} = O(\varepsilon^{-\f{d}{r}})$ and $\N_{\Phi} = O( \varepsilon^{-\f{d}{r}})$, such that
	$\|f-\Phi\|_{p,w}\leq \varepsilon$.
	
	Moreover, the constant factor of $\N_{\Phi}$ only depends on $d,r$ and can be at most $CC_rd$, where the constant $C>0$ and the constant $C_r$ only depends on $r$.

            \item[\rm{(ii)}] Let $f$ from the unit ball of $ W^r_\infty([0,1]^d)$.

            For any $\varepsilon>0$, there exists a function $\Phi$ realized by a SignReLU neural network with $\L_{\Phi}=O(1)$ and $\N_{\Phi} = O( \varepsilon^{-\f{d}{r}})$, such that
	$\|f-\Phi\|_\infty \leq \varepsilon$.
        \end{enumerate}
	
\end{thm}

The improvement on the constant factor $C_rd$ is significant in the $L_{p,w}$ case. In fact, to achieve the same accuracy $\varepsilon\in(0,1)$ under $L_\infty$ norm, Theorem 1 of \cite{Yarosky} asserts that
	$f\in W^r_\infty([0,1]^d)$ can be approximated by a ReLU deep
	net with at most $c(\ln(\varepsilon^{-1})+1)$ layers and at most $c (\varepsilon^{-\f{d}{r}}\ln(\varepsilon^{-1})+1)$ computation
	units with a constant $c := c_{d, r}$. However, the constant $c$ here
	increases much faster as $d$ becomes large. More specifically, as pointed out in \cite{Zhou,suh}, the main approach in \cite{Yarosky} is to
	approximate $f$ by a localized Taylor polynomial, which leads to the constant $c$ at least $2^d$ when $d$ is large. 

    The following theorem shows the rate obtained by using SignReLU for approximating Sobolev functions is optimal. We denote $B_1(F)$ the unit ball of any given function class $F$ centered at $0 \in F$.
	\begin{thm}[\cite{devore1989optimal}]\label{de}
		Let $d,r\in \NN_+$. Let $\N>0$ be an integer and $\Psi:\RR^\N \rightarrow C([0,1]^d)$ be an arbitrary mapping. Assume that there is a continuous map $w : B_1(W^r_\infty)\rightarrow \RR^\N$ such that $\| f- \Psi(w(f)) \|_\infty \leq \varepsilon$ for all $f \in B_1(W^r_\infty([0,1]^d))$. Then $\N \geq c_r \varepsilon^{-\f{d}{r}}$ with $c_r$ be a constant only depends on $r$.
	\end{thm}
	
	When we fix $r$ and $d$, \thmref{sobolev} and \thmref{de} show that the obtained bounds $\N_{\Phi}$ of SignReLU neural networks that achieve tolerance $\varepsilon$ can not be improved under the hypothesis of continuous weight selection.



\subsection{Approximation of piecewise smooth functions}\label{psf}

\begin{figure}[tbp]
	\centering
	\subfigure[$f_1(x,y) = e^{x+y-1}$]{\includegraphics[width=5.5cm]{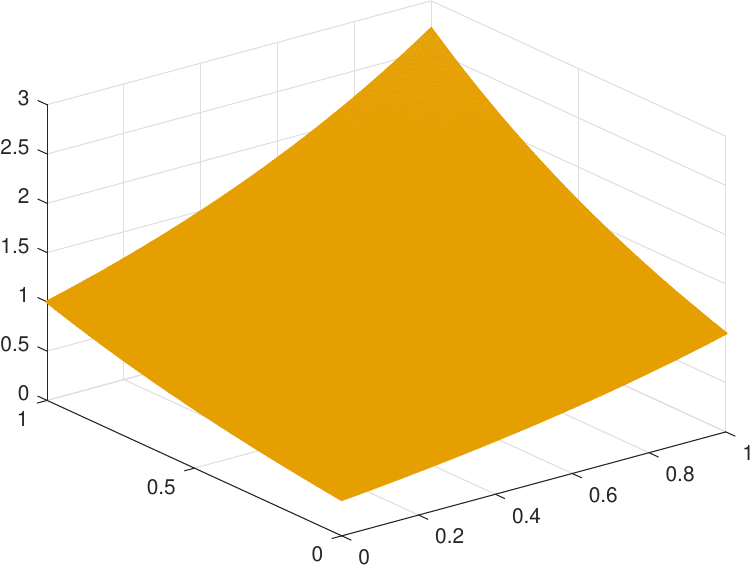}}\quad
	\subfigure[$f_2(x,y) = 2(x-0.5)^2 + 2(y-0.5)^2-3$]{\includegraphics[width=5.5cm]{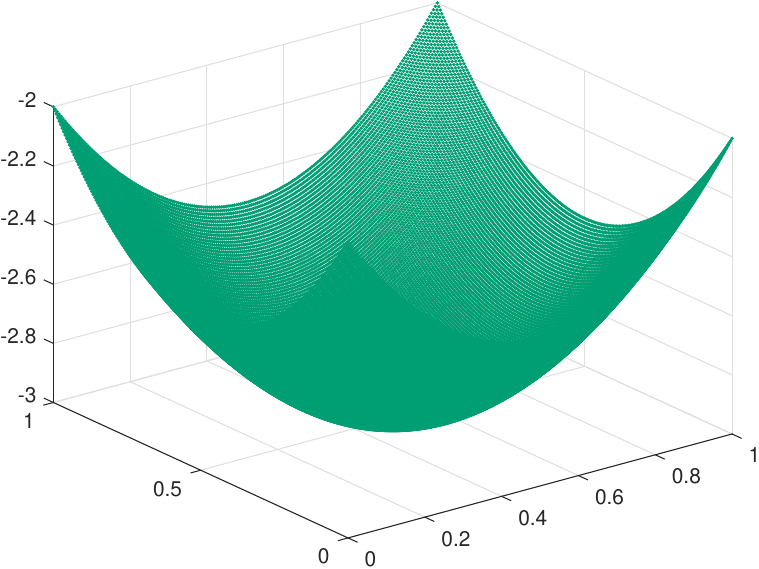}}
	\\
	\subfigure[$\chi_\Omega$]{\includegraphics[width=5.5cm]{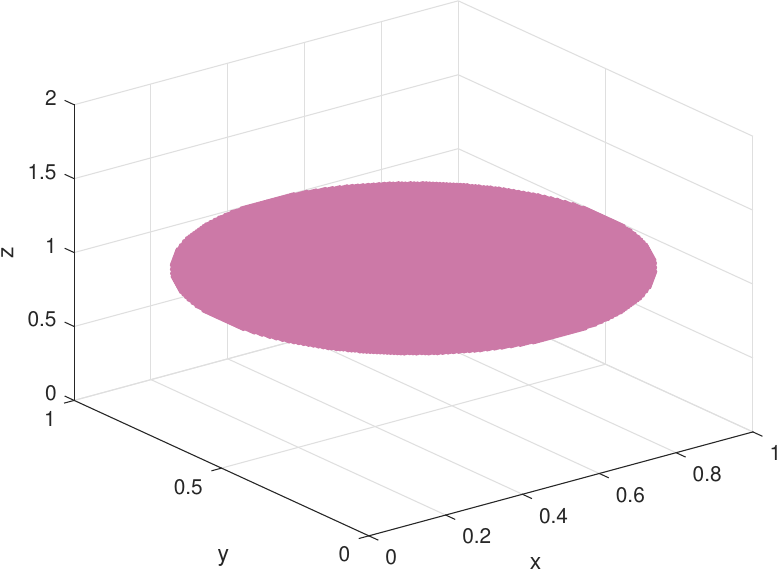}}\quad
	\subfigure[$f_1 + f_2\chi_\Omega$]{\includegraphics[width=5.5cm]{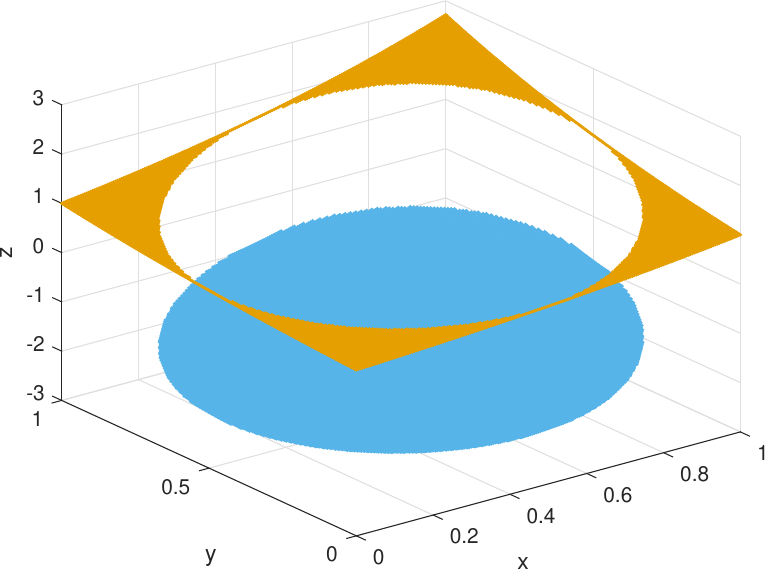}}
	\caption{Visualization of $f_1$, $f_2$, $\chi_\Omega$ and $f = f_1 + f_2 \chi_{\Omega}$ with the set $\Omega$ defined as $\Omega = \left\{ (x,y)\in[0,1]^d: (x-0.5)^2+(y-0.5)^2 < 0.25\right\}$. The orange/green/purple/blue region represents the value of the function $f_1$/$f_2$/$1$/$f_1+f_2$ on the corresponding region.  }
	\label{pieces}
\end{figure}

In this subsection,
we consider piecewise functions.  Given a function space $F$ on $[-1,1]^d$ and a collection $A$ of subsets of $[-1,1]^d$, the collection of piecewise functions is defined as
\[S(F, A) = \{f_1 + f_2 \chi_\Omega: f_1, f_2 \in F, \Omega \in A, \|f_2\|_{L_\infty\left([-1,1]^d\right)}\leq 1   \}.\]

We consider the collection $A$ as a collection of level sets by
\[ \Omega = \left \{\b x \in [-1,1]^d: h(\b x)< g(\b x) \right  \}  ,\]
where $h,g$ are SignReLU neural networks that can $\varepsilon$-approximate some functions in $F$. 
Results in \thmref{sobolev} show that using networks to define $A$ will not influence its generality. If the function space $F$ is defined on $[0,1]^d$, then in $S(F,A)$, the collection $A$ will be modified accordingly.

For example, we choose $F = W^r_\infty\left([0,1]^2\right)$, $g(x,y) = -(x-0.5)^2-(y-0.5)^2$, and $h(x,y)\equiv -0.25$. By \lemref{ra}, functions $h$ and $g$ can be realized by some SignReLU neural networks. Hence, we have $\Omega = \left\{ (x,y)\in[0,1]^d: (x-0.5)^2+(y-0.5)^2 < 0.25\right\}$, which is a region bounded by a circle. Let $f_1(x,y) = e^{x+y-1}$ and $f_2(x,y) = 2(x-0.5)^2 + 2(y-0.5)^2-3$ and define $f := f_1 + f_2 \chi_{\Omega}$. Then obviously, $f/3  \in S(F,A)$. See Figure~\ref{pieces} for illustration. In fact, any choice of $f_1, f_2 \in F$ makes $f \in S(F,A)$.

Since the collection $A$ depends on the function class $F$, we use the notation $S(F):=S(F, A)$ for short, and when functions in $F$ have some smooth properties, we call functions in $S(F)$ piecewise smooth functions induced by $F$. 

\begin{thm}\label{piecewise smooth}
	Let $d,r \in \NN_+$ and $p\geq 1$.
\begin{enumerate}
    \item [\rm{(i)}] Let $f \in S\left(B_1\left(W^r_p([-1,1]^d,w)\right) \right)$.

    For any $\varepsilon>0$, there exist a set $\Omega_\varepsilon \subset [-1,1]^d$ and a function $\Phi$ realized by a SignReLU neural network with $\L_{\Phi} = O(\varepsilon^{-\f{1}{r}})$, $\W_{\Phi} = O(\varepsilon^{-\f{d}{r}})$ and $\N_{\Phi} = O(\varepsilon^{-\f{d}{r}})$ such that
		\begin{equation*}
		\|f-\Phi\|_{p,w,[-1,1]^d\backslash \Omega_\varepsilon }\leq \varepsilon,
		\end{equation*}
        and $|\Omega_\varepsilon |\rightarrow 0$, as $\varepsilon \rightarrow 0 $.

        \item [\rm{(ii)}] Let $f \in S\left(B_1\left(W^r_\infty([0,1]^d)\right) \right)$.

    For any $\varepsilon>0$, there exist a set $\Omega_\varepsilon \subset [0,1]^d$ and a function $\Phi$ realized by a SignReLU neural network with $\L_{\Phi} = O(1)$ and $\N_{\Phi} = O(\varepsilon^{-\f{d}{r}})$ such that
		\begin{equation*}
		\|f-\Phi\|_{\infty,[0,1]^d\backslash \Omega_\varepsilon }\leq \varepsilon,
		\end{equation*}
        and $|\Omega_\varepsilon |\rightarrow 0$, as $\varepsilon \rightarrow 0 $.

\end{enumerate}

\end{thm}

\thmref{piecewise smooth} is easy to be extended from a binary classification setting $(f_1 + f_2)\chi_{\Omega} + f_1\chi_{\Omega^c}$ to multi-classification setting $\sum_{i=1}^{K}f_i\chi_{\Omega_i}$, with width and the total number of weights growing at a rate of several times, which only depends on $K$.
 Since we have $B_1\big(W^r_\infty ([0,1]^d)\big) \subset S\left(B_1\big(W^r_\infty ([0,1]^d)\big)\right)$, \thmref{piecewise smooth} cannot be improve under the hypothesis of \thmref{de}.

\subsection{Approximation with milder dependence on dimensionality}\label{BV}

For approximating Sobolev functions with regularity $r$ and a pre-assigned accuracy $\varepsilon$, the required size of SignReLU nets is $ O(\varepsilon^{-\f{d}{r}})$, which increases exponentially with respect to input dimension $d$. In real-world image applications, the dimension $d$ of an image is usually larger than $100\times100\times 3$. To achieve tolerance $\varepsilon = 0.1$, the size of neural networks is almost $10^{30000/r} $, which is only applicable when $r$ is very large. In this subsection, we attempt to break the curse of dimensionality in approximating multivariate functions by employing a ``tensor-friendly'' structure. 

 Here breaking the curse of dimensionality means the approximation rate or complexity rate of a model can be merely impacted by the dimension of inputs. To achieve it, we will introduce a space of functions with finite mixed derivative norms, which is sometimes referred to as a Korobov space.
Related investigations with ReLU activation have been achieved in \cite{montanelli2019new, ZhouMao}.

The ability of using Korobov space to overcome the curse of dimensionality is followed from mixed derivatives and hyperbolic approximations \cite{dung2018hyperbolic}. When using polynomials for approximating Sobolev functions, there will be $\# \{ \b n \in \NN^d: \|\b n\|_\infty  \leq N \} \approx N^d$ multi-dimensional monomials required. Denote $\|\b n\|_{\pi}=\prod_{j=1}^d \max\{1, n_j\} $. Then only $\# \{ \b n \in \NN^d: \|\b n\|_\pi  \leq N \} \approx N(\ln N)^{d-1}$ monomials are involved for approximating Korobov functions, which increase almost of order $N$, instead of $N^d$.

The weighted Korobov space $K^r_{p}([-1,1]^d, w)$ with $r\in \NN$ and $1\leq p<\infty$ is defined by locally integrable
functions on $ [-1,1]^d$ with  norm
\[\|f\|_{K^r_{p}}:= \sum_{\|\b k\|_\infty\leq r}\left[\int_{[-1,1]^d}|D^k f(\b x)|^p w(\b x)d\b x\right]^{1/p} <\infty. \]
The approximation theory related to Korobov space was also considered in \cite{montanelli2019new, ZhouMao,dung2018hyperbolic}.

The following theorem shows that the dominant term of complexity rate is free of the input dimension.

\begin{thm}\label{koborov}
	Let $d, r\in \NN_+$, $ p \geq 1$ and $\beta(r,d)=(2dr+d+r)/r$.
	\begin{enumerate}
		\item Let $f$ from the unit ball of $K^r_p([-1,1]^d,w)$.

            For any $\varepsilon>0$, there exists a function $\Phi$ realized by a SignReLU neural network with $\L_{\Phi} = O\left((\log_2 \varepsilon^{-1})\right)$ and $\N_{\Phi} =O\left(\varepsilon^{-\f{1}{r}}(\log_2 \varepsilon^{-1})^{\beta(r,d)}\right) $ such that
		$\|f-\Phi\|_{p,w} \leq \varepsilon$.
		
		\item Let $f \in S\left(B_1(K^r_p([-1,1]^d)) \right)$.

            For any $\varepsilon>0$, there exist a set $\Omega_\varepsilon \subset [-1,1]^d$ and a function $\Phi$ realized by a SignReLU neural network with $\L_{\Phi} = O\left(\log_2 (\varepsilon^{-1})\right)$ and $\N_{\Phi} = O\left(\varepsilon^{-\f{1}{r}}(\log_2 \varepsilon^{-1})^{\beta(r,d)}\right)$ such that $\|f-\Phi\|_{p,w,[-1,1]^d\backslash \Omega_\varepsilon }\leq \varepsilon$, and $|\Omega_\varepsilon |\rightarrow 0$, as $\varepsilon \rightarrow 0 $.
	\end{enumerate}
\end{thm}

Sparse grids are employed to break the curse of dimensionality of ReLU neural networks for approximating functions from Koborov space \cite{montanelli2019new} for $r=2$. \thmref{koborov} extends the results to $r\neq 2$ and piecewise Korobov functions.

The above result can be further improved by restricting functions in the form of rank-one tensor. For $r = 1,2,\dots$, let $V^r [0,1]$ be the space of functions that satisfy $g$, $g'$, $\dots$, $g^{(r-1)}$ are absolutely continuous on $[0,1]$ and $g^{(r)}$ is of bounded variation not exceeding $1$. 
We denote a set of functions on $[0,1]^d$ as
\begin{align*}
	\mathcal{V}^r_d =& \Big\{f(\b x) = \prod_{i=1}^{d} f_i(x_i): f_i \in V^{r_i} [0,1], \\
	& \quad \|f_i\|_\infty \leq 1,
	 r_i \leq r \Big\}.
\end{align*}

\begin{thm}\label{BVfunctions}
	Let $d,r\in \NN_+$. For any $\varepsilon>0$ and $f \in \mathcal{V}^r_d$, there exists a function $\Phi$ realized by a SignReLU neural network with $\L_{\Phi} = O( \varepsilon^{-\f{1}{r+1}} )$, $\W_{\Phi} = O(1)$ and $\N_{\Phi} = O( \varepsilon^{-\f{1}{r+1}} )$ such that $\|f-\Phi\|_{\infty} < \varepsilon$.
\end{thm}

The improvement comes from the classical rational approximation theory.

\subsection{Convergence rate of approximating continuous function}\label{continuous}
SignReLU neural networks are able to approximate general continuous functions. The approximation error will be estimated in terms of the modulus of continuity. For $t>0$, we define the modulus of continuity as
\[ w_f^i(t):= \sup\Big\{|f(\b x) - f(\b y)|: | x_i -  y_i| \leq t, \forall j\neq i ,x_j = y_j, \b x, \b y \in  [0,1]^d \Big \}, \]
which can be bounded by the classical modulus $w_f(t):= \sup\{|f(\b x) -f( \b y)|: \b x, \b y \in  [0,1]^d , \|\b x - \b y\|_2 \leq t \}$.

\begin{thm}\label{continuous1}
	Let $d,N\in \NN_+$. For any continuous function $f$ on $[0,1]^d$, there exists a function $\Phi$ realized by a SignReLU neural network with $\L_{\Phi} = O(N)$, $\W_{\Phi} = O(N^d)$ and $\N_{\Phi} = O(N^d)$ such that $||\Phi - f||_{L_{\infty}\left([0,1]^d\right)}\leq \f{5}{4} \sum_{i=1}^{^d} w_f^i(\f{1}{N})$.
\end{thm}

\begin{figure*}[ht]
	\centering
	\subfigure[]{\includegraphics[width=5cm]{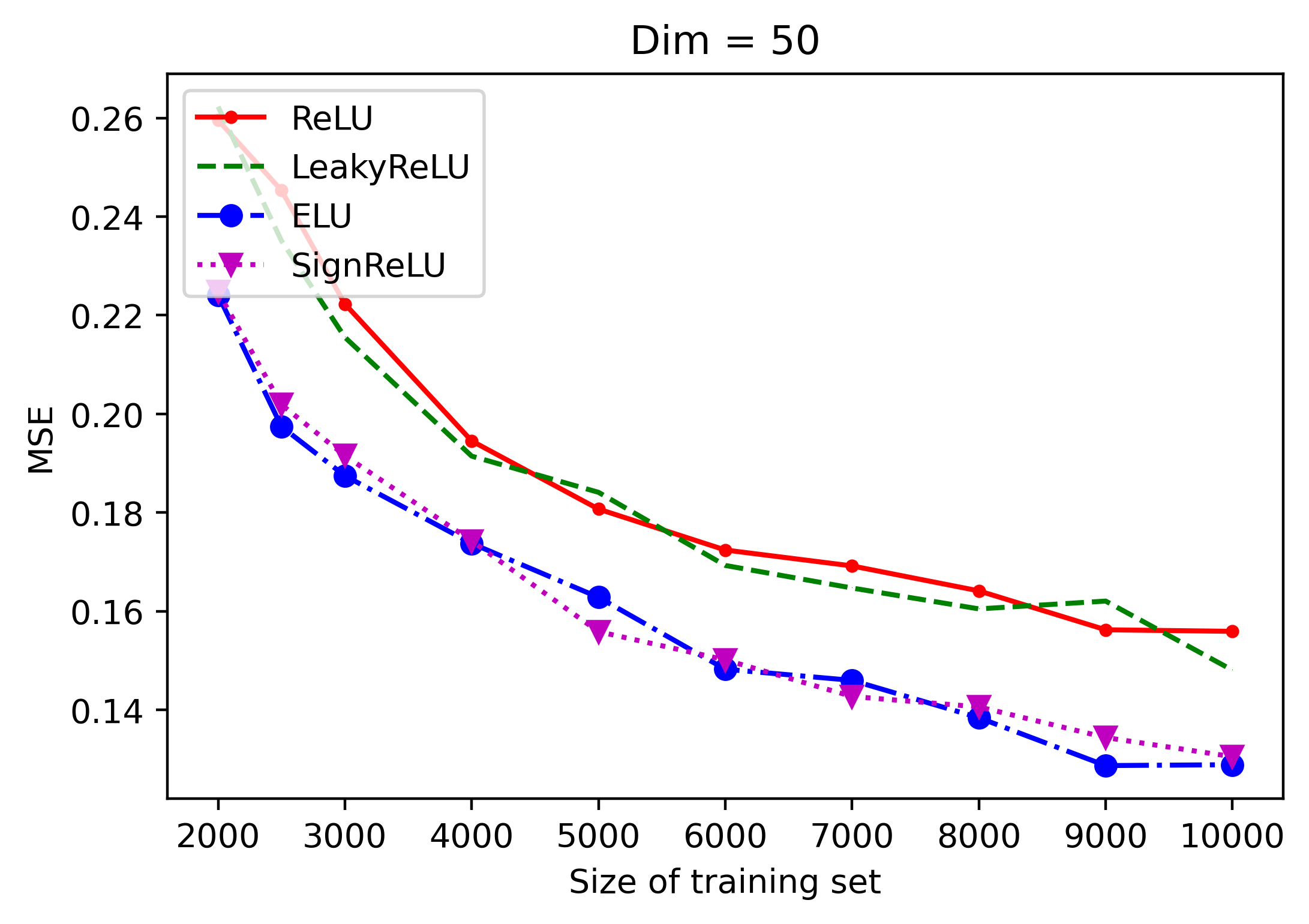}}
	\subfigure[]{\includegraphics[width=5cm]{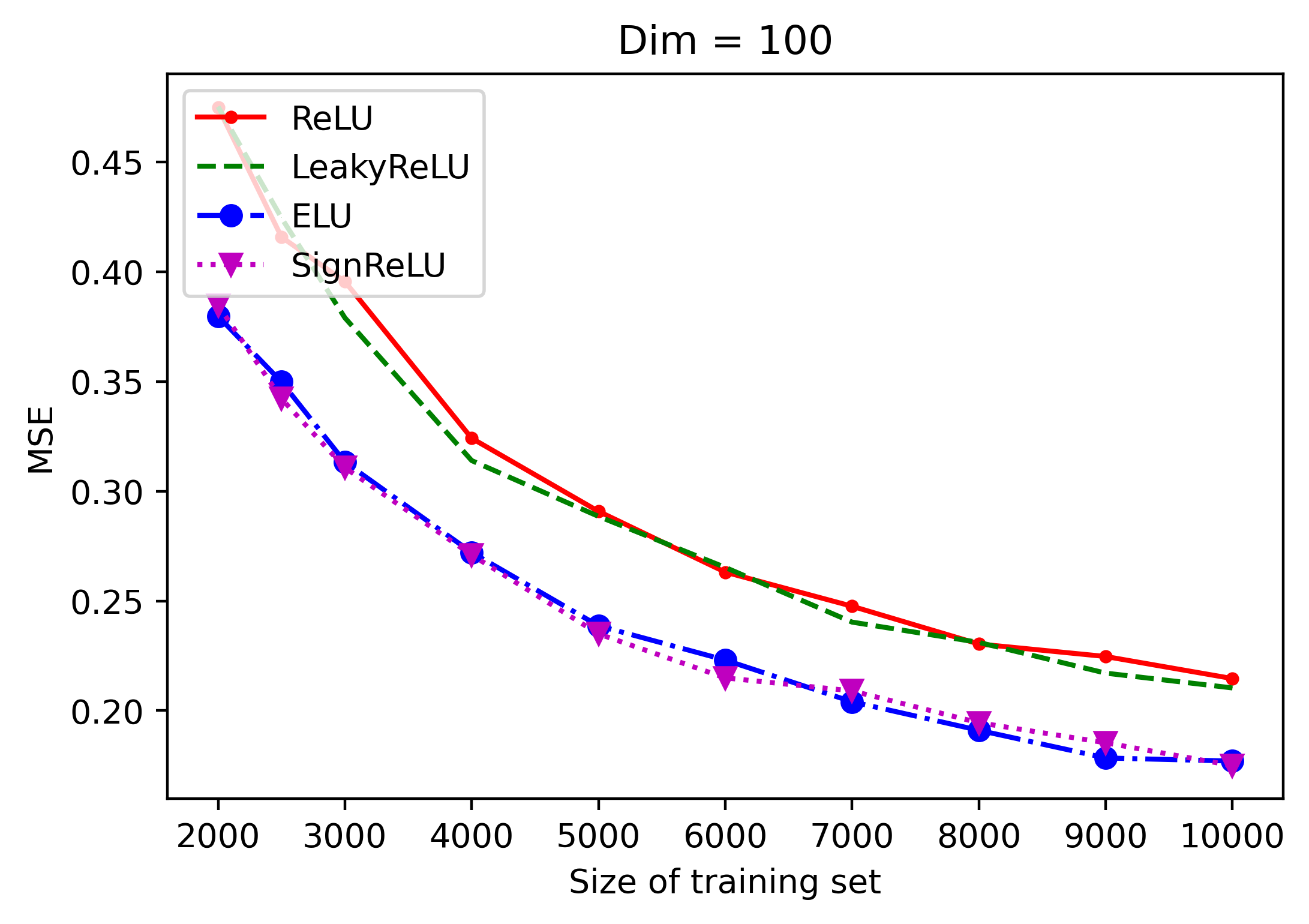}}
	\subfigure[]{\includegraphics[width=5cm]{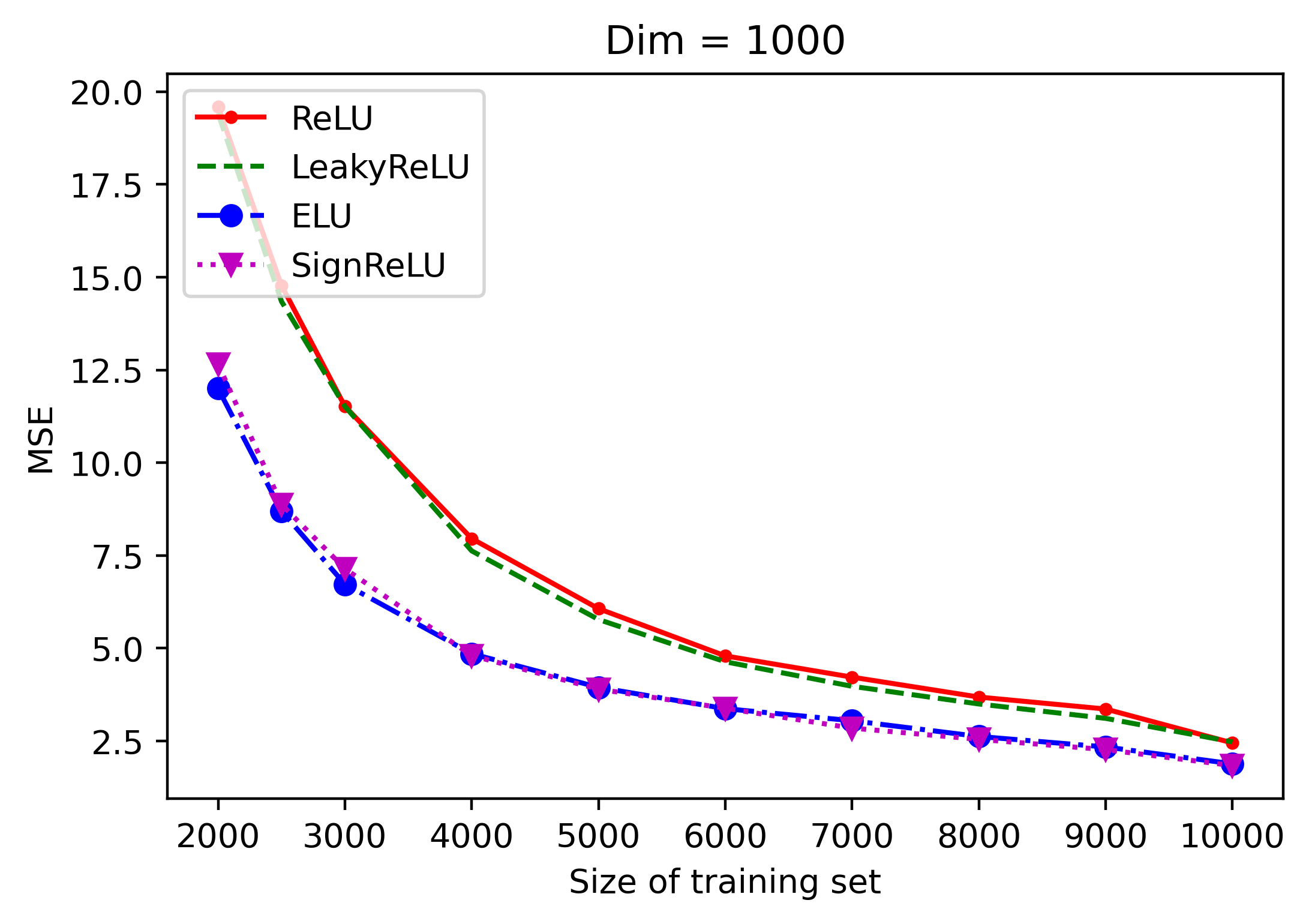}}
	\\
	\subfigure[]{\includegraphics[width=5cm]{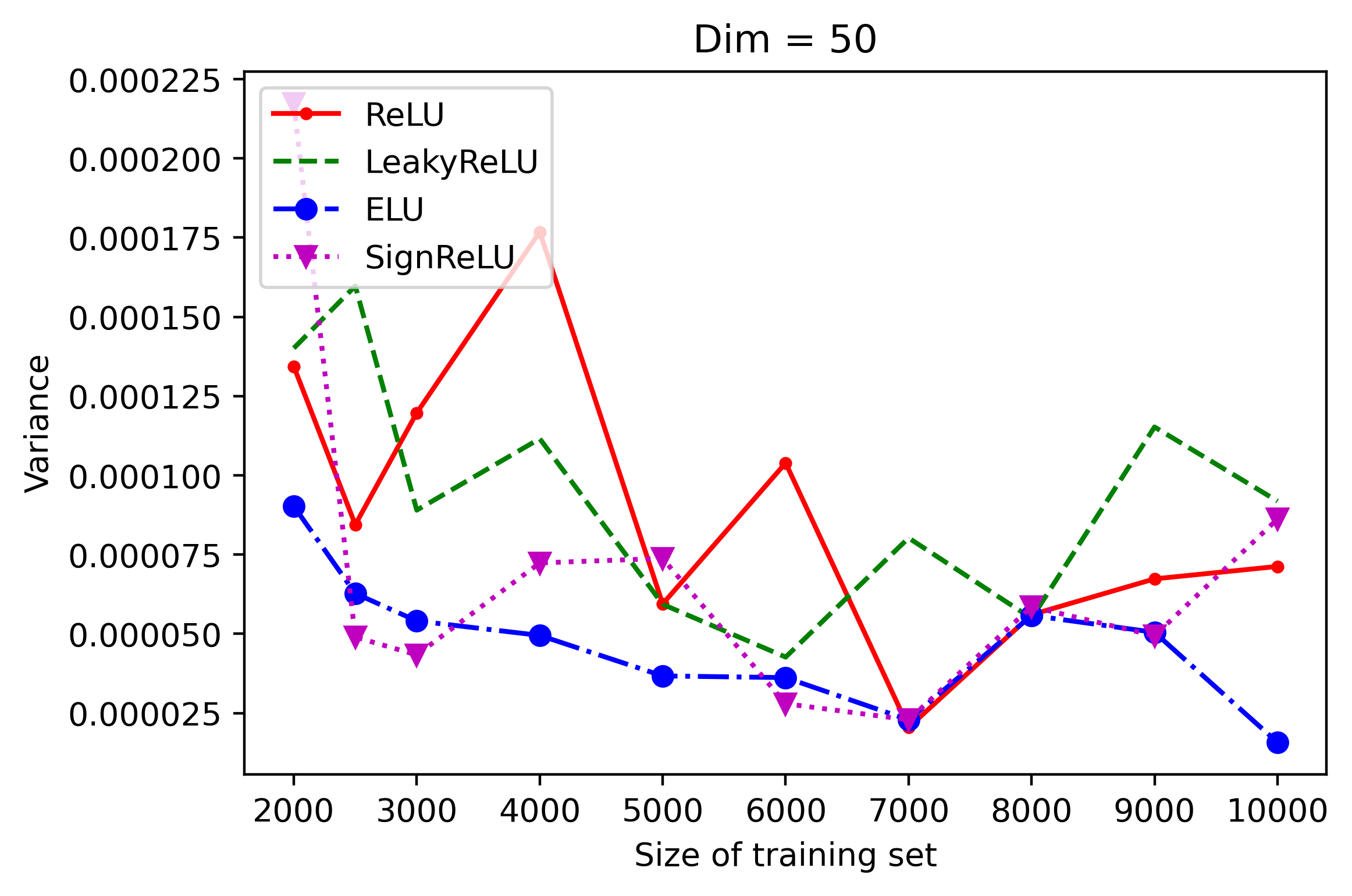}}
	\subfigure[]{\includegraphics[width=5cm]{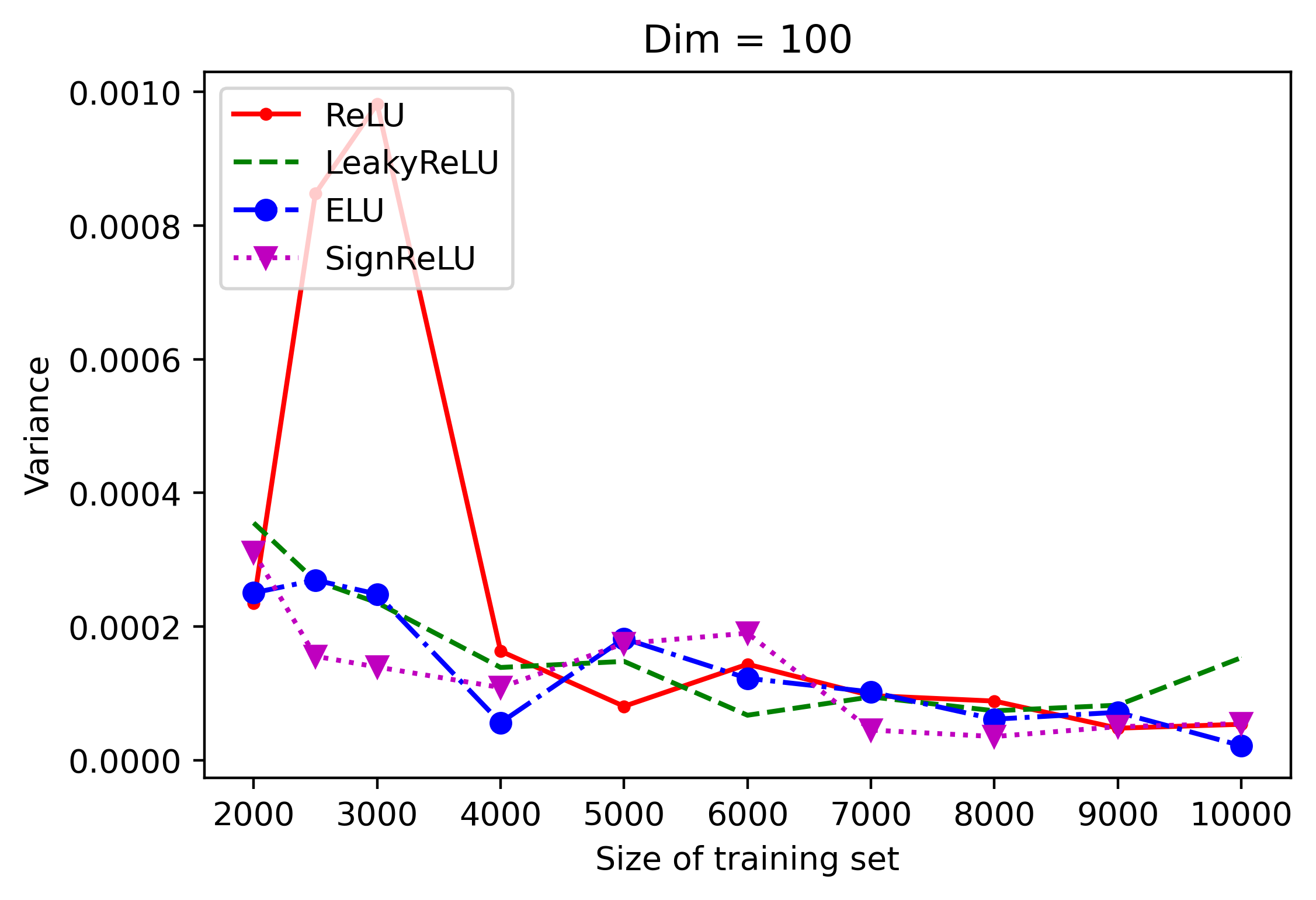}}
	\subfigure[]{\includegraphics[width=5cm]{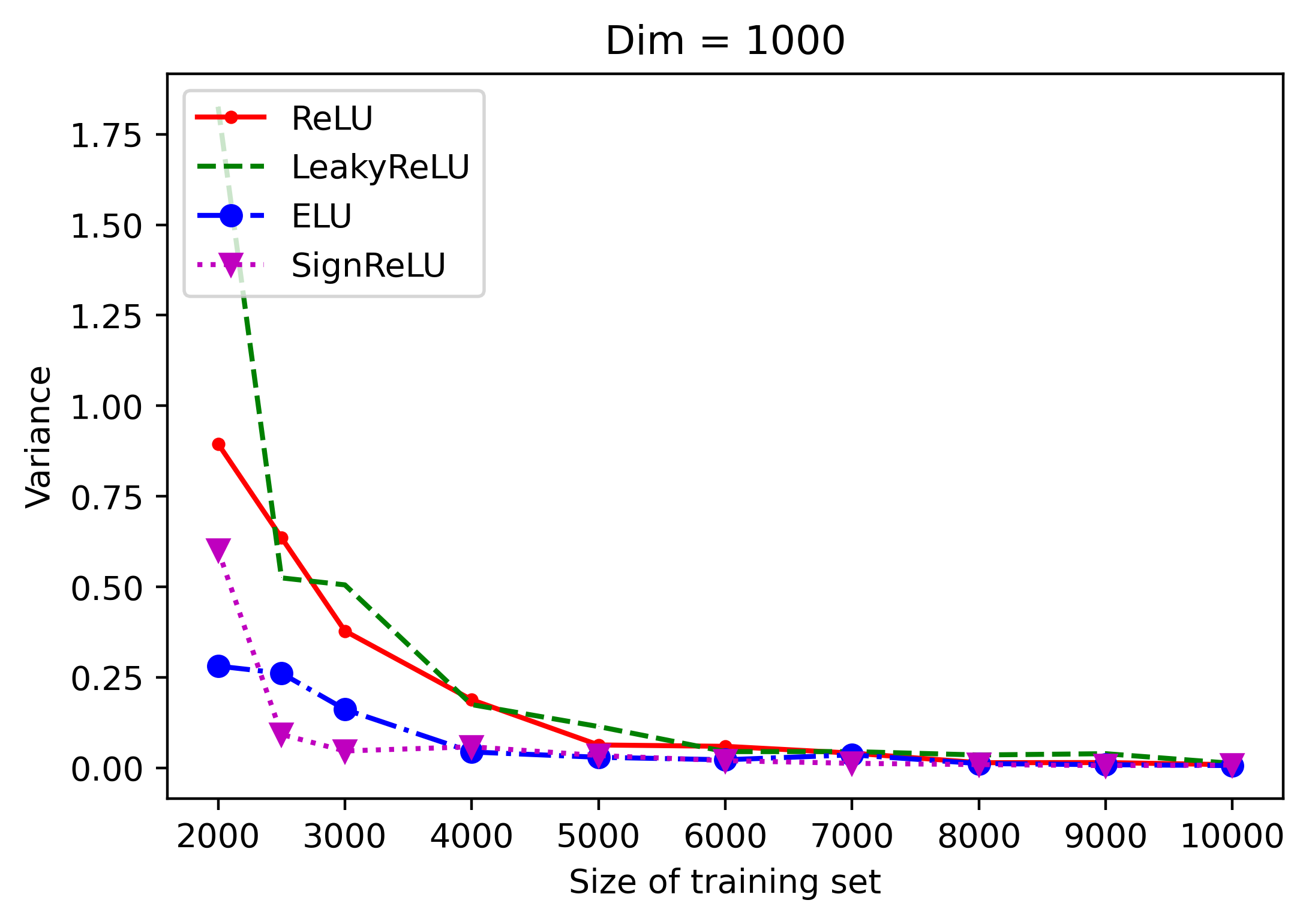}}
	\caption{Noise regression results. Figures (a-c), (d-f) show MSEs and variances on the test sets of a three-layer fully connected neural network with different activation functions. X-axis represents the size of the training set and Y-axis gives mean MSE and variance over $10$ independent trails.}
	\label{reg_test}
\end{figure*}

\begin{figure*}[ht]
	\centering
	\subfigure[]{\includegraphics[width=5cm]{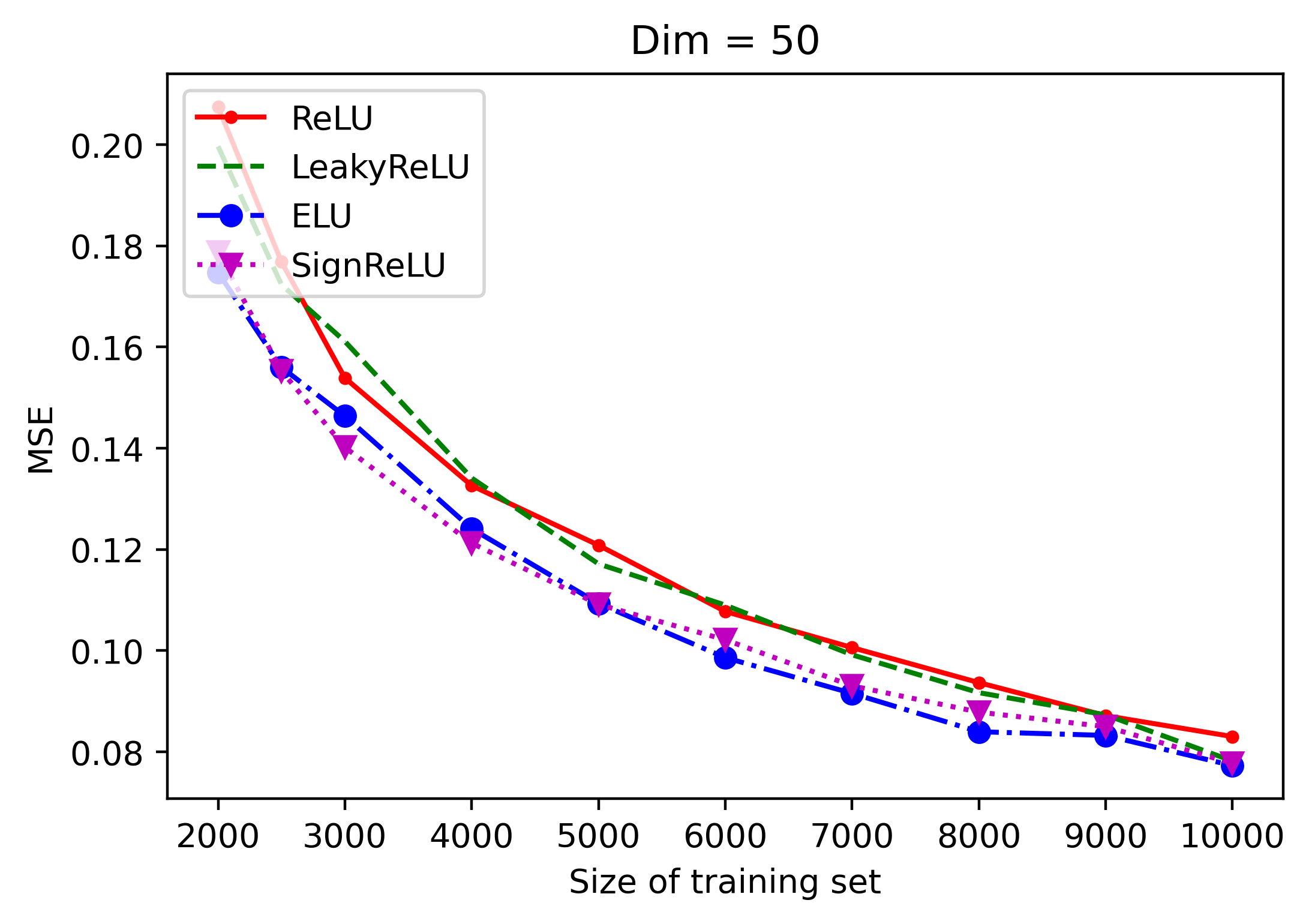}}
	\subfigure[]{\includegraphics[width=5cm]{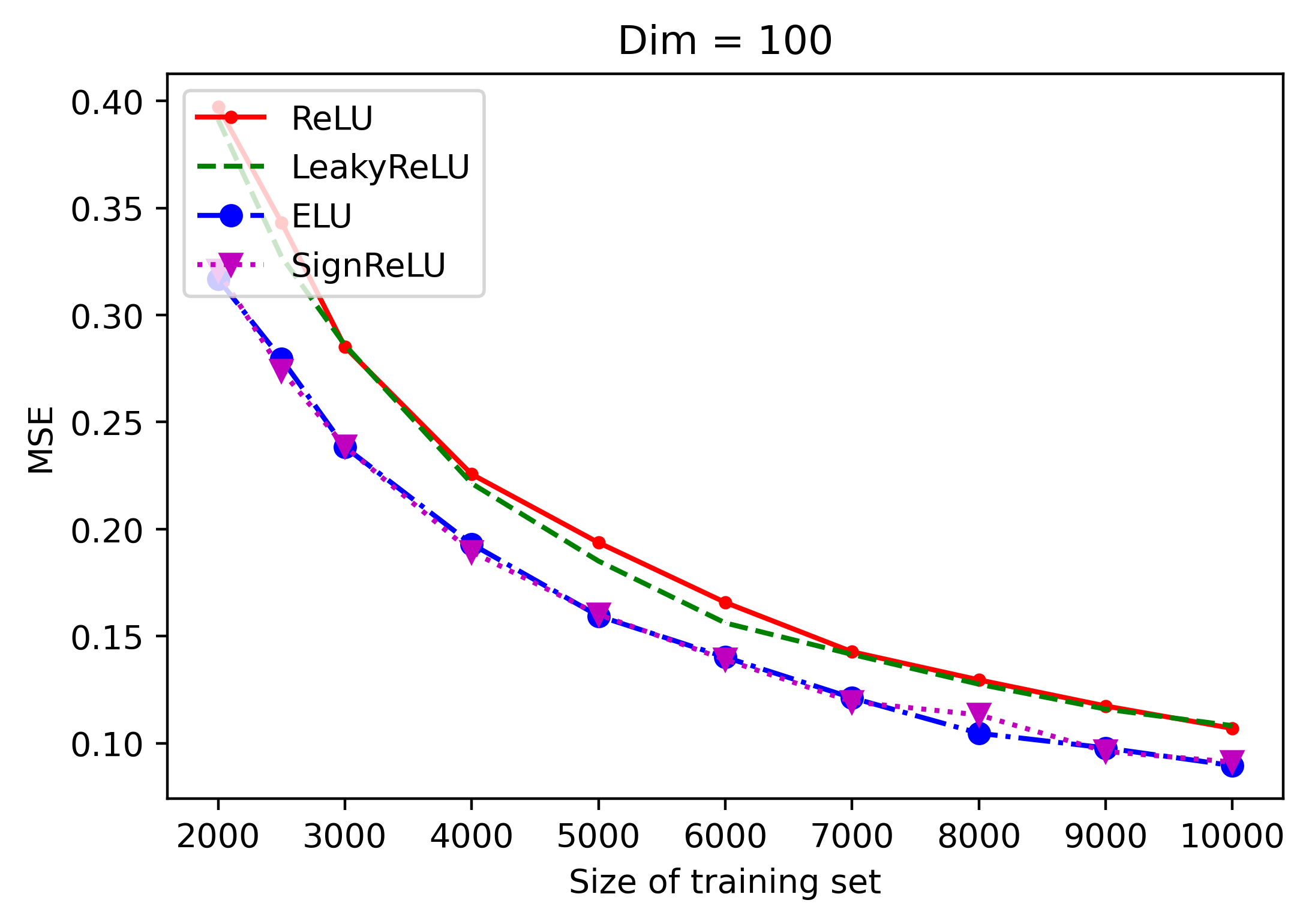}}
	\subfigure[]{\includegraphics[width=5cm]{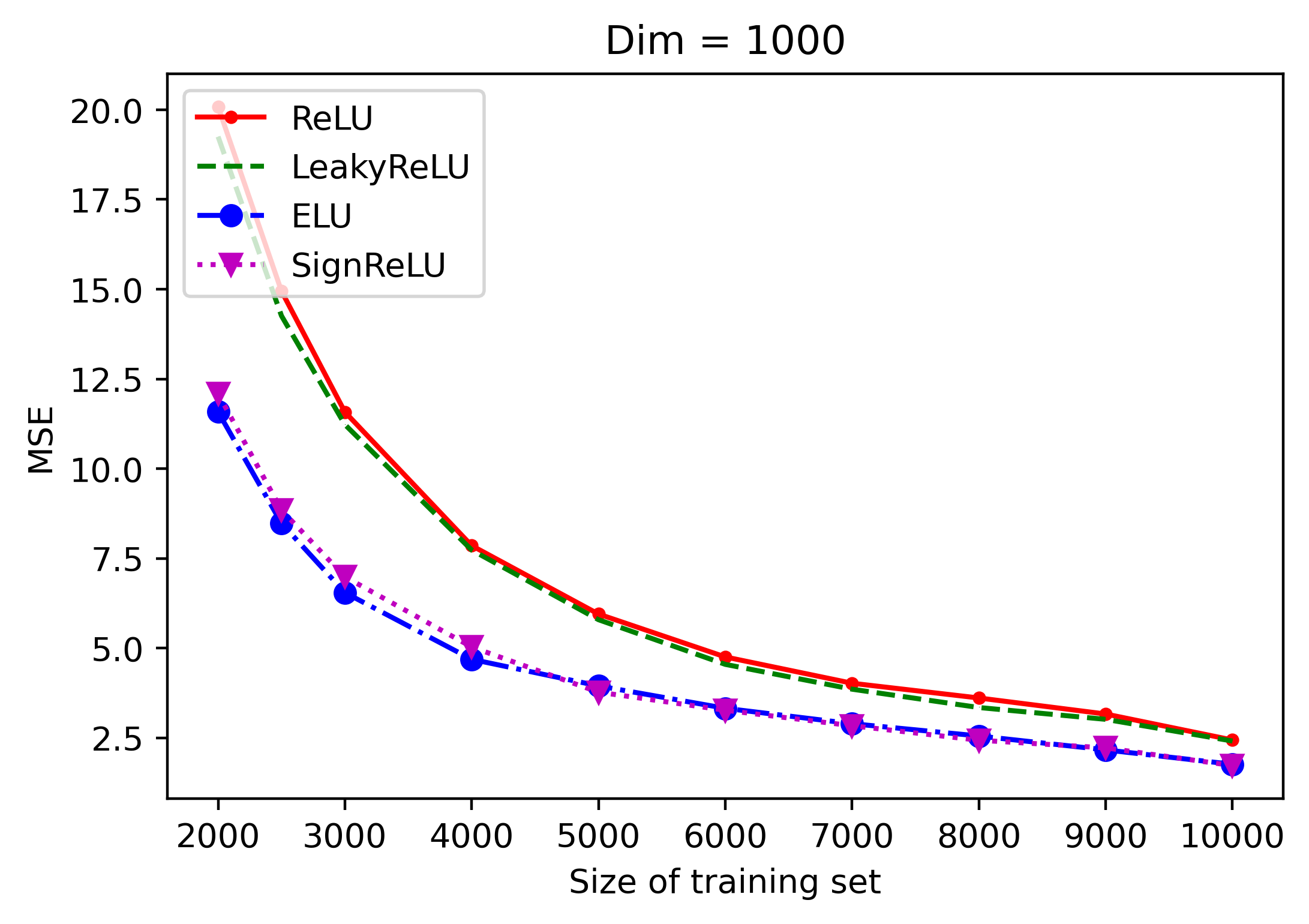}}
	\\
	\subfigure[]{\includegraphics[width=5cm]{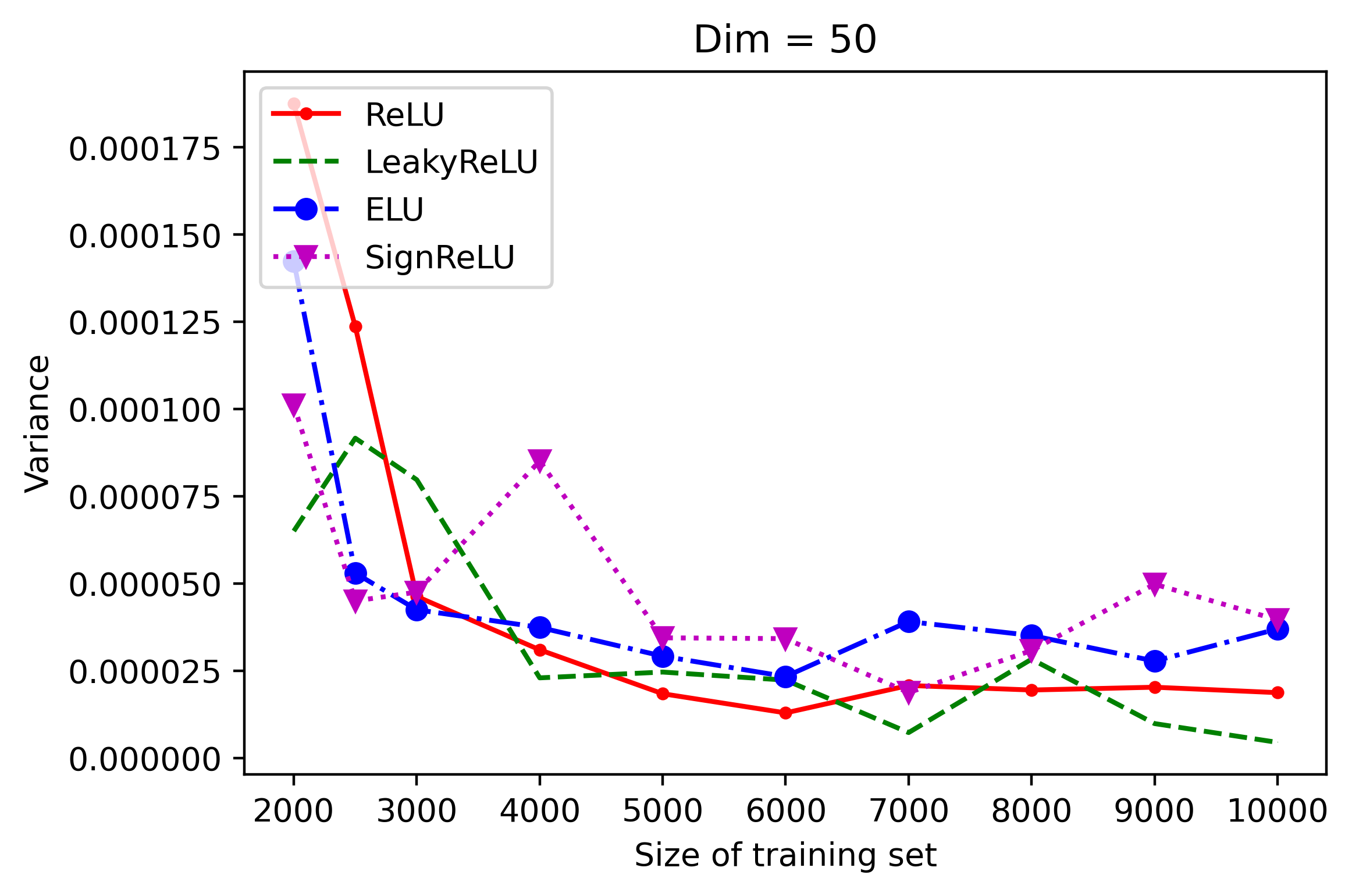}}
	\subfigure[]{\includegraphics[width=5cm]{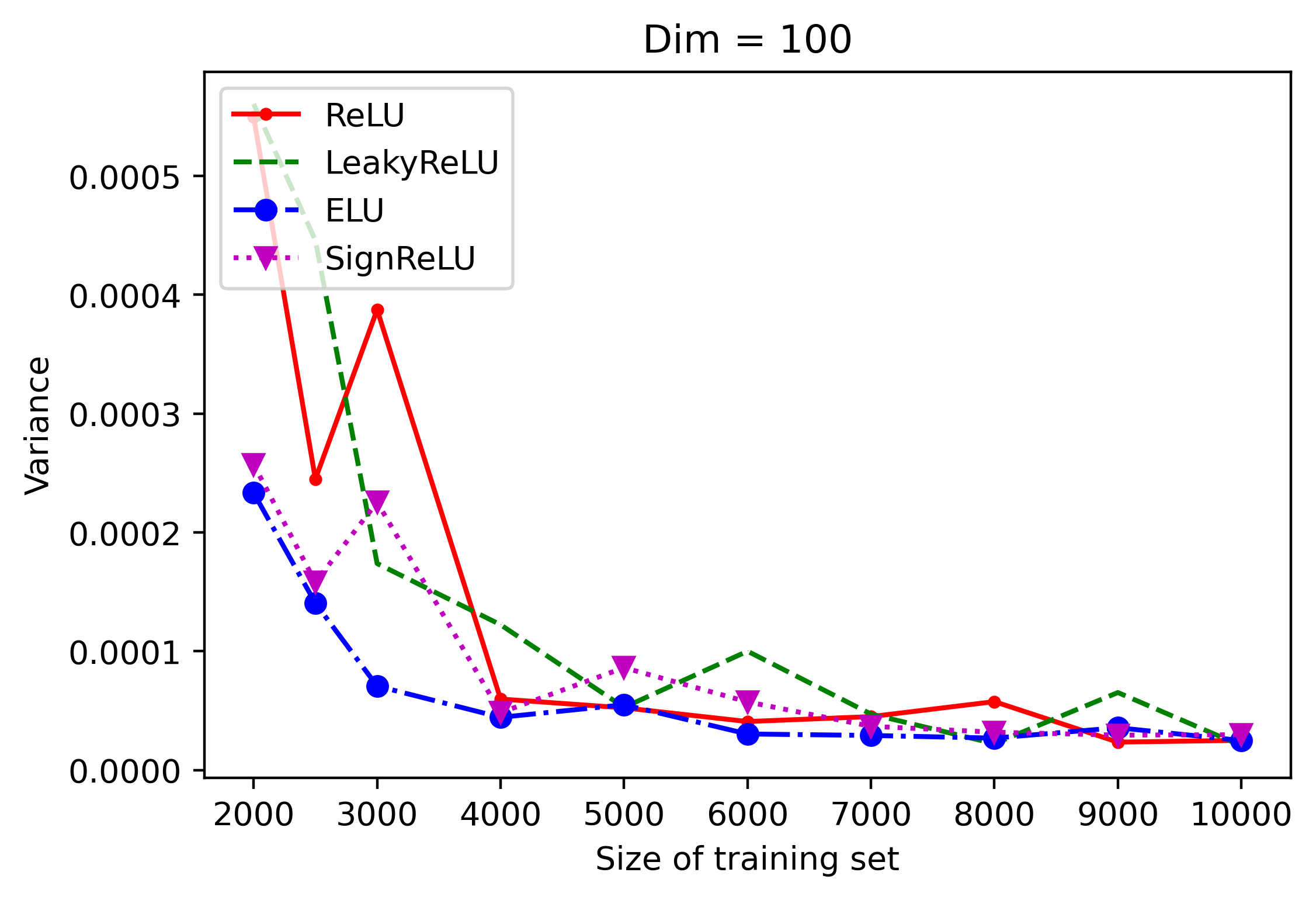}}
	\subfigure[]{\includegraphics[width=5cm]{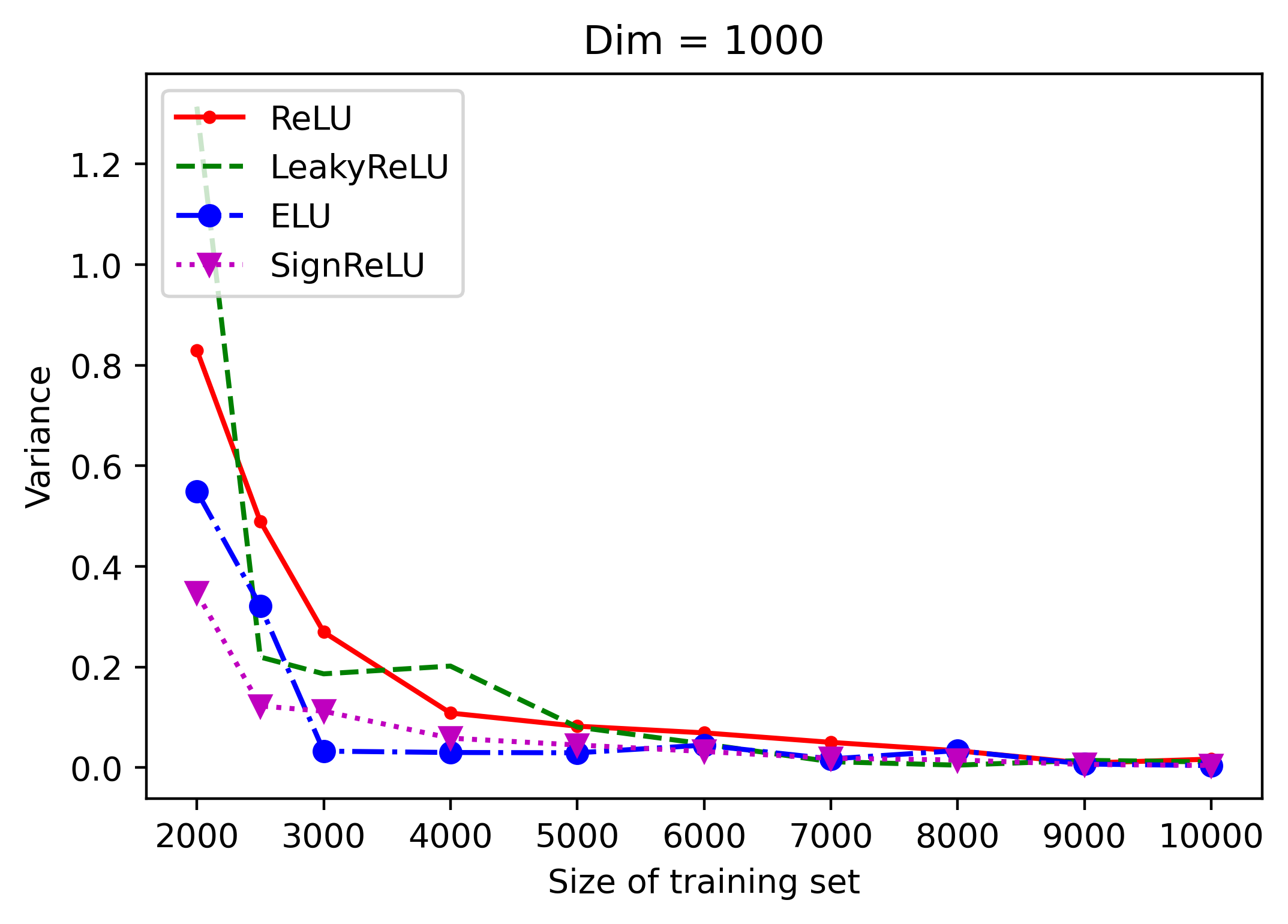}}
	\caption{Noiseless regression results. Figures (a-c), (d-f) show MSEs and variances on the test sets of a three-layer fully connected neural network with different activation functions. X-axis represents the size of the training set and Y-axis gives mean MSE and variance over $10$ independent trails.}
	\label{reg_test_noiseless}
\end{figure*}

\section{Experiments}
In this section, we conduct some numerical experiments to test the ability of the SignReLU activation function for various learning tasks (regression, classification, image denoising). Codes are available at  \url{https://github.com/JFLi11/Experiments-on-SignReLUnets.git}. In each experiment, we utilize a neural network with the same architecture activated by different nonlinear activation functions, ReLU, LeakyReLU ($\alpha=0.01$), ELU ($\alpha=1$), and SignReLU. The datasets we take include MNIST \cite{MNIST} images with 60000 images for training and 10000 for testing, CIFAR10 \cite{cifar10} images with 50000 for training and 10000 for testing, and Caltech101 \cite{caltech101} images with 7677 for training and 1000 for testing.

\subsection{Regression}
\begin{figure}[H]
	\centering
	\subfigure[Random noised points]{\includegraphics[width=3.5cm]{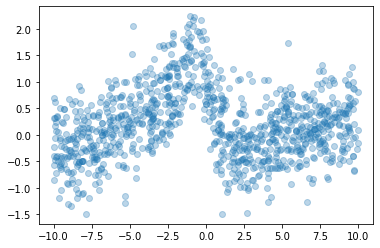}}
	\quad \subfigure[Random ground truth points]{\includegraphics[width=3.5cm]{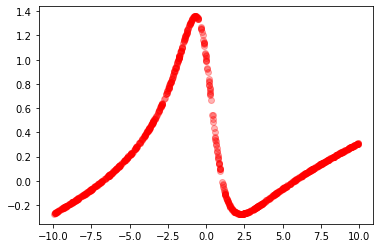}}
	\caption{Regression model with 1000 random points on $[-10,10]$.}
	\label{reg}
\end{figure}

We consider the following regression model on $  [-10,10]^d$
\[y = \f{2 - 2z + 0.05z^3}{2+z^2} + \xi, \]
where $z = \sum_{i=1}^{d}x_i$ and $\xi$ is the Gaussian noise with mean $0$ and variance $0.25$. See an illustration of the proposed regression model in Figure~\ref{reg}. We draw data with $\b x$ uniformly sampled with dimension $d$ varying in $\{50,100,1000\}$ and the size of the training set varying in \[\{2000,2500,3000,4000, \dots, 10000\}\] for each $d$. We randomly generate $2000$ samples for the test set in the same way as the training set without noise for all experiments in this subsection. For the network structure, we choose three hidden layers with widths all equal to $100$.
During the training, we use the ADAM algorithm and a minibatch of size $100$. The learning rate is set to be $0.0001$ in $50$ epochs.

Figure~\ref{reg_test} depicts the results of $10$ independent trials in terms of MSE. We observe that ELU and SignReLU outperform ReLU and LeakyReLU and that ELU and SignReLU have similar performances for this regression model. When more data are used for training, all activations can help the neural network learn well. If no noise is added in the training sets, the performance of neural networks, see Figure~\ref{reg_test_noiseless}, can be improved and other conclusion is similar to noise cases.

\subsection{Classification}
\begin{table}[H]
	\centering
	
	\scriptsize
	\setlength\tabcolsep{2.5pt}
	\linespread{1.9}
	
	\caption{Test accuracy.}
	\label{acc_te}
	\begin{tabular}{|c|c|c|c|c|}
		\hline
		\textbf{Test   acc(\%)} & \textbf{ReLU} & \textbf{SignReLU}   & \textbf{ELU} & \textbf{LeakyReLU} \\ \hline
		MNIST                   & 97.82         & \textbf{97.92} & 97.12         & 97.43              \\ \hline
		CIFAR10                 & 75.28         & \textbf{76.57} & 76.07        & 74.62              \\ \hline
	\end{tabular}
\end{table}

In the classification task \cite{feng,zhuxn}, we evaluate activation functions on MNIST and CIFAR10. For MNIST, we use a fully connected neural network which is the same as that used in the previous subsection on regression. During the training, we use the ADAM algorithm and a minibatch of size $128$. The learning rate is set to be $0.001$ in $20$ epochs.
For CIFAR10, we use a "small" Resnet18 with output channels for each convolutional layer to be 16, and max-pooling is applied after each residual block. Since the image is small, we also remove the first $7\times7$ kernel of Resnet18. The classification layer we employed is a fully connected layer with the input dimension $64$. This neural network has no more than $40$ thousand parameters.
During the training, we use the ADAM algorithm and a minibatch of size $128$. The learning rate decays exponentially from the beginning value $0.001$ with the multiplicative factor of $0.9$ every four epochs in $30$ epochs. Weight decay is set to be $0.001$.
Test accuracy in Table~\ref{acc_te} proves the superiority of the classification algorithm induced by SignReLU neural networks.

\begin{figure}[t]
	\centering
	\includegraphics[width=4.5in]{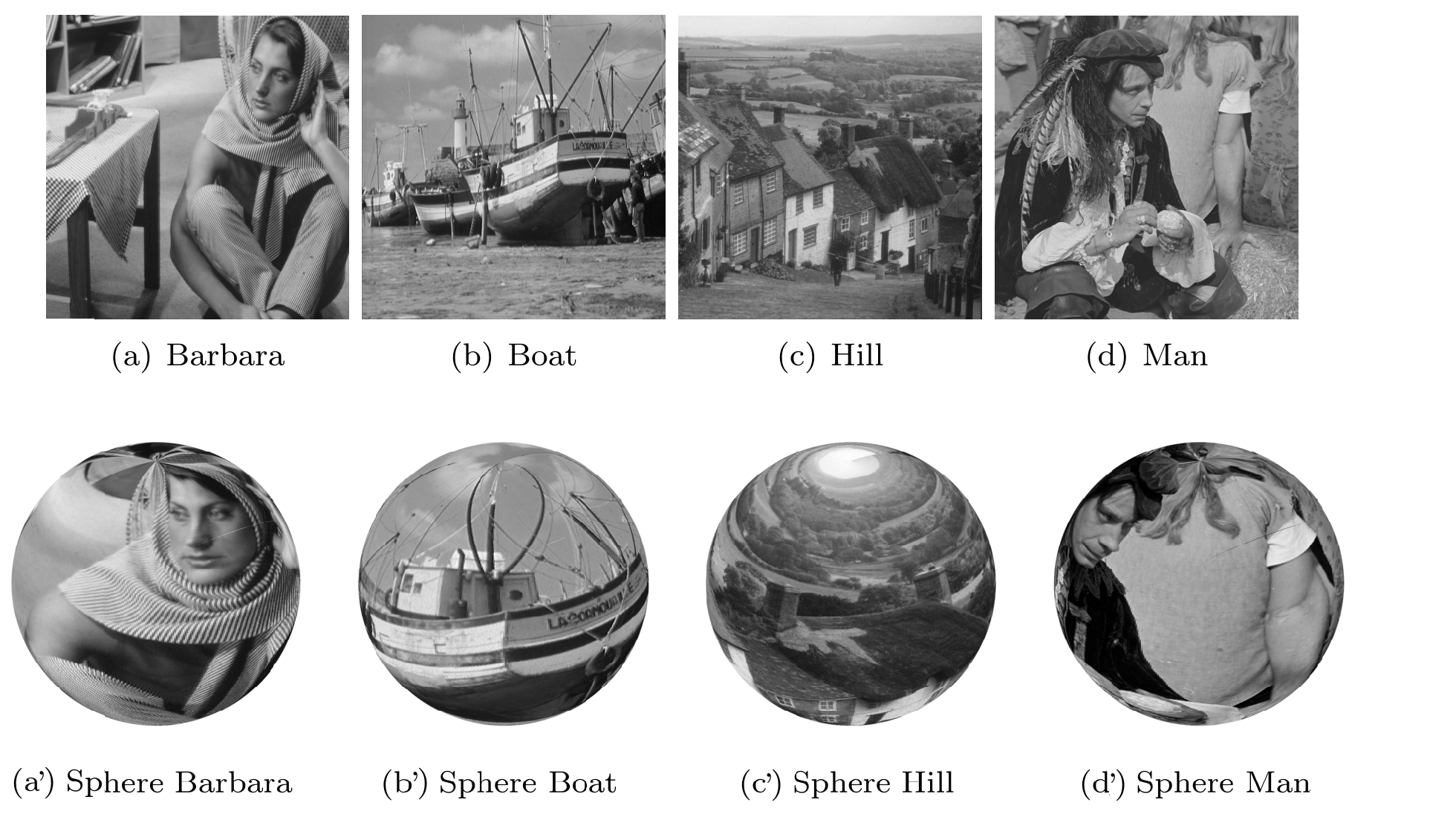}\\
	\caption{2D gray images and their corresponding spherical samples \cite{li2022convolutional}.}\label{sph2Dimg}
\end{figure}
\subsection{Spherical image denosing}

\begin{table}[b]
	\centering
	\scriptsize
	\setlength\tabcolsep{1.5pt}
	\linespread{1.2}
	
	\setlength{\tabcolsep}{2pt} 
	\renewcommand{\arraystretch}{2} 
	\caption{Average PSNRs on four images by trained NN over 5 independent trials .}
	\label{sphimg}
	\begin{tabular}{|c|lll|lll|lll|lll|}
		\hline
		\textbf{Image} & \multicolumn{3}{c|}{\textbf{Barbara}}                                                                           & \multicolumn{3}{c|}{\textbf{Boat}}                                                                              & \multicolumn{3}{c|}{\textbf{Hill}}                                                                              & \multicolumn{3}{c|}{\textbf{Man}}                                                                               \\ \hline
		rate           & \multicolumn{1}{c|}{\textbf{0.2}}    & \multicolumn{1}{c|}{\textbf{0.3}}    & \multicolumn{1}{c|}{\textbf{0.5}} & \multicolumn{1}{c|}{\textbf{0.2}}    & \multicolumn{1}{c|}{\textbf{0.3}}    & \multicolumn{1}{c|}{\textbf{0.5}} & \multicolumn{1}{c|}{\textbf{0.2}}    & \multicolumn{1}{c|}{\textbf{0.3}}    & \multicolumn{1}{c|}{\textbf{0.5}} & \multicolumn{1}{c|}{\textbf{0.2}}    & \multicolumn{1}{c|}{\textbf{0.3}}    & \multicolumn{1}{c|}{\textbf{0.5}} \\ \hline
		ReLU           & \multicolumn{1}{l|}{\textbf{23.823}} & \multicolumn{1}{l|}{22.533}          & 21.309                            & \multicolumn{1}{l|}{\textbf{26.104}} & \multicolumn{1}{l|}{24.515}          & 22.777                            & \multicolumn{1}{l|}{\textbf{26.034}} & \multicolumn{1}{l|}{24.703}          & 23.116                            & \multicolumn{1}{l|}{\textbf{26.367}} & \multicolumn{1}{l|}{24.901}          & 23.243                            \\ \hline
		LeakyReLU      & \multicolumn{1}{l|}{23.760}          & \multicolumn{1}{l|}{22.470}          & \textbf{21.324}                   & \multicolumn{1}{l|}{25.870}          & \multicolumn{1}{l|}{24.569}          & \textbf{22.942}                   & \multicolumn{1}{l|}{26.013}          & \multicolumn{1}{l|}{24.709}          & \textbf{23.216}                   & \multicolumn{1}{l|}{26.262}          & \multicolumn{1}{l|}{24.953}          & \textbf{23.403}                   \\ \hline
		ELU            & \multicolumn{1}{l|}{23.626}          & \multicolumn{1}{l|}{22.504}          & 21.319                            & \multicolumn{1}{l|}{25.750}          & \multicolumn{1}{l|}{24.582}          & 22.746                            & \multicolumn{1}{l|}{25.762}          & \multicolumn{1}{l|}{24.725}          & 23.163                            & \multicolumn{1}{l|}{26.037}          & \multicolumn{1}{l|}{25.000}          & 23.333                            \\ \hline
		SignReLU            & \multicolumn{1}{l|}{23.805}          & \multicolumn{1}{l|}{\textbf{22.589}} & 21.152                            & \multicolumn{1}{l|}{25.967}          & \multicolumn{1}{l|}{\textbf{24.657}} & 22.554                            & \multicolumn{1}{l|}{26.001}          & \multicolumn{1}{l|}{\textbf{24.730}} & 23.058                            & \multicolumn{1}{l|}{26.343}          & \multicolumn{1}{l|}{\textbf{25.010}} & 23.134                            \\ \hline
	\end{tabular}
\end{table}

One more experiment we conduct follows that in \cite{li2022convolutional} for spherical image denoising with convolutional neural network activated by ReLU. Spherical images, whose domain are 2-sphere, similar to graph-structured data, are not typical images on Euclidean space but arise in various situations, such as astrophysics \cite{starck2006wavelets} and medical imaging \cite{yu2007cortical}. One of the most famous spherical image tasks is to process CMB data (Cosmic Microwave Background radiation field) to obtain useful information \cite{abrial2008cmb}.

Here we employ the neural network with the same architecture activated by various functions including LeakyReLU, ELU, and SignReLU. During the training, we use the ADAM algorithm and a mini-batch size of $20$. Learning rate decay exponentially from the beginning value $0.005$ with a multiplicative factor $0.9$ in $20$ epochs. For any image $\b f$, we add Gaussian noise with varying standard deviation $\sigma = \text{rate} \times f_{max}$ where $f_{max}$ is the maximal absolute value of $\b f$. See Figure~\ref{reg_noise} for some noisy images. The peak signal-to-noise ratio (PSNR $=10\log_{10}(f^2_{max}/\text{MSE})$) is employed to evaluate the performance of each denoising model, where $\text{MSE}$ represents the mean square error between noised signal and ground truth.

The generalization performance is evaluated by recovered PSNR over four typical images (sampled to the 2D sphere, see Figure~\ref{sph2Dimg} for illustration and \cite{li2022convolutional} for more details), while the neural network is trained on Caltech101, see Table~\ref{sphimg} for results. We find that the SignReLU activation function can give the best-denoised image when $\text{rate}=0.3$ while ReLU and LeakyReLU activations perform best when given noise rate $0.2$ and $0.5$ respectively. With noise rate increases, all activations have a decreasing performance. Overall, under the above settings, all these four activations give comparable denoised images. Hence, the SignReLU activation function can be valuable for image-denoising tasks. Figure~\ref{reg_noise} shows some denoising results of neural networks. As the noise rate increases, neural networks cannot recover image textures well. For example, facial parts are not able to be seen at noise rate $0.5$. Developing neural networks that improve PSNR and structure information simultaneously would be an interesting topic for future work. 


\begin{figure}[htpb!]
	\centering
	\subfigure[Noise rate $ = 0.2$]{\includegraphics[width=3.7cm]{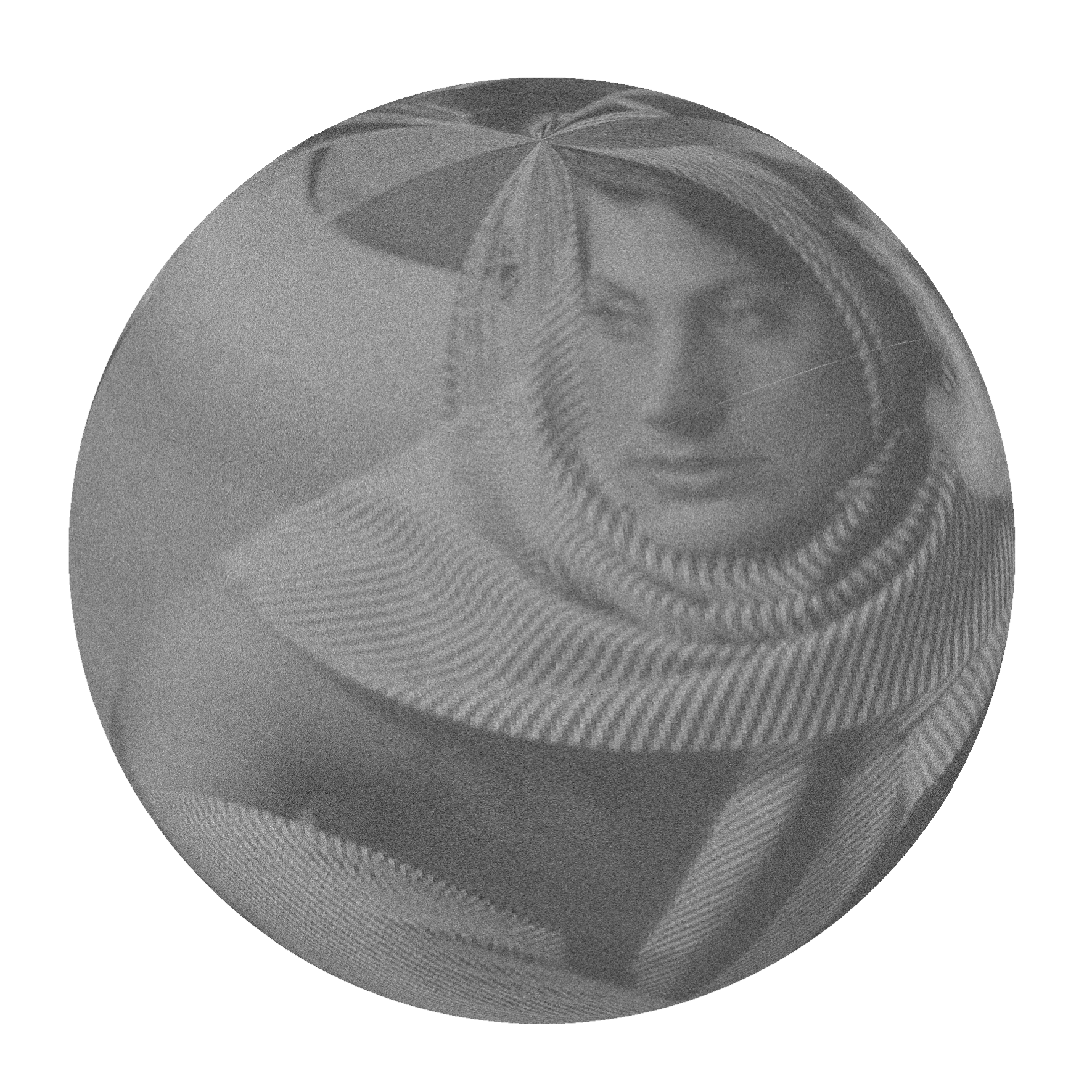}}
	\quad 
	\subfigure[Noise rate $ = 0.3$]{\includegraphics[width=3.7cm]{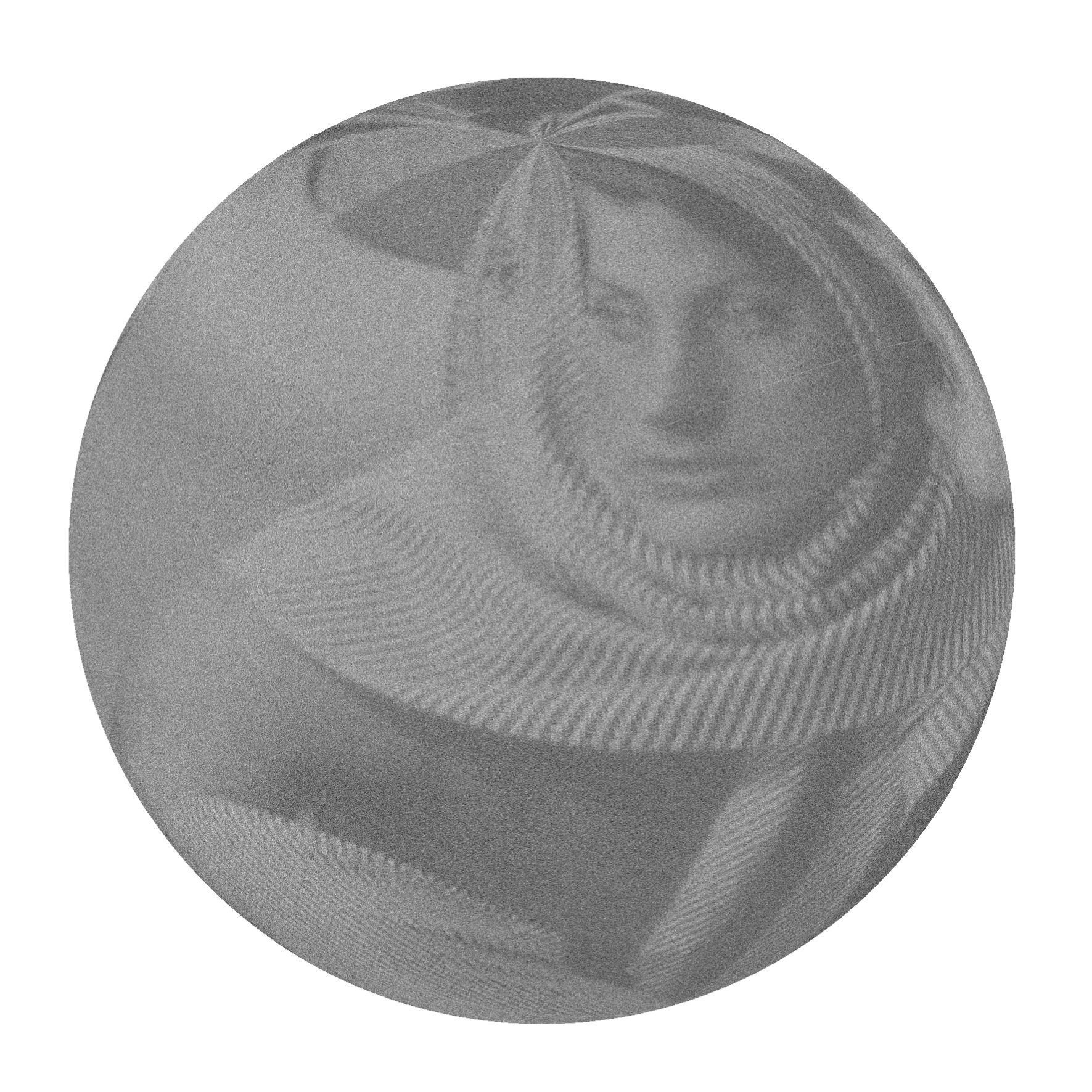}}
	\subfigure[Noise rate $ = 0.5$]{\includegraphics[width=3.7cm]{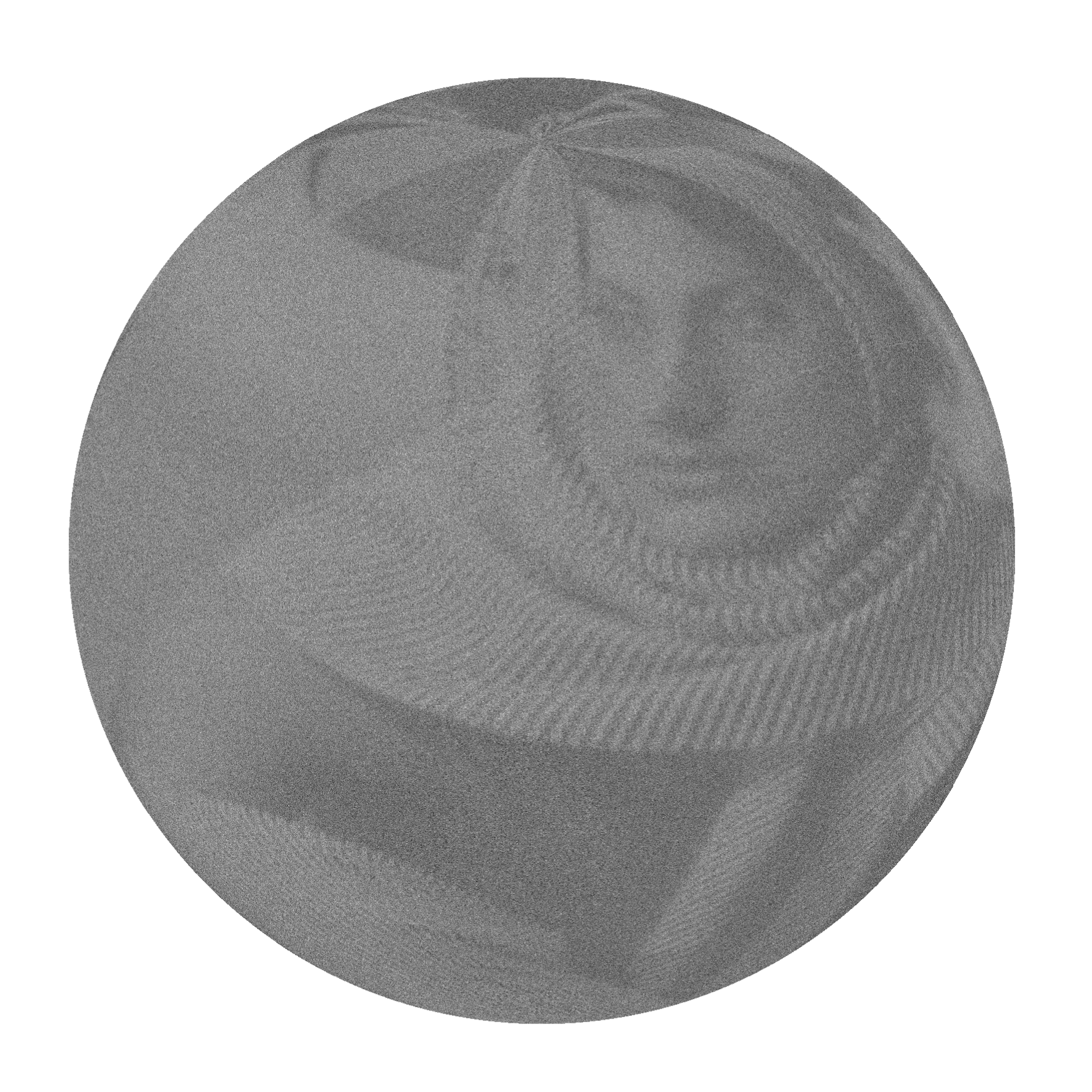}}
	
	\subfigure[Denoised image $  0.2$]{\includegraphics[width=3.7cm]{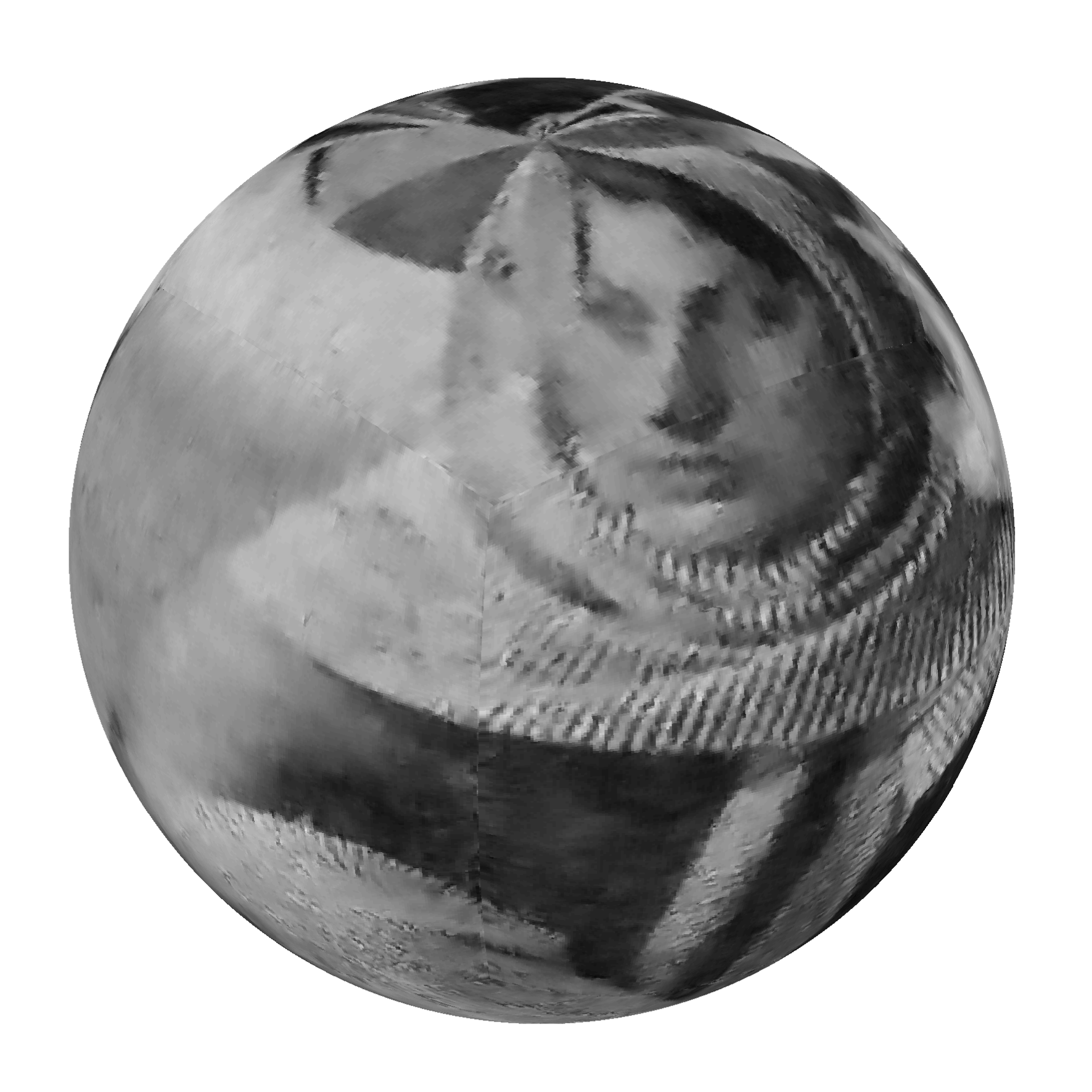}}
	\quad 
	\subfigure[Denoised image $  0.3$]{\includegraphics[width=3.7cm]{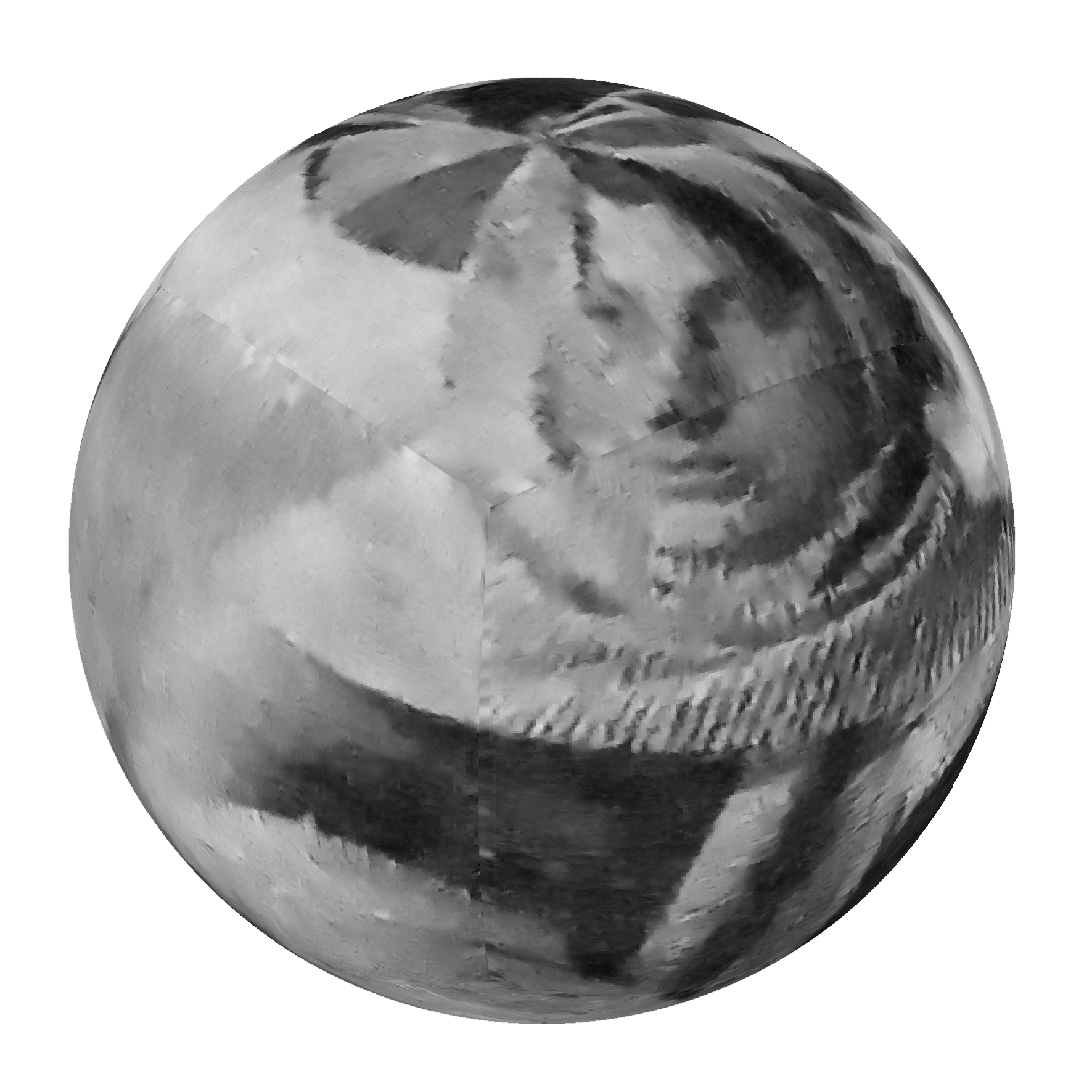}}
	\subfigure[Denoised image $ 0.5$]{\includegraphics[width=3.7cm]{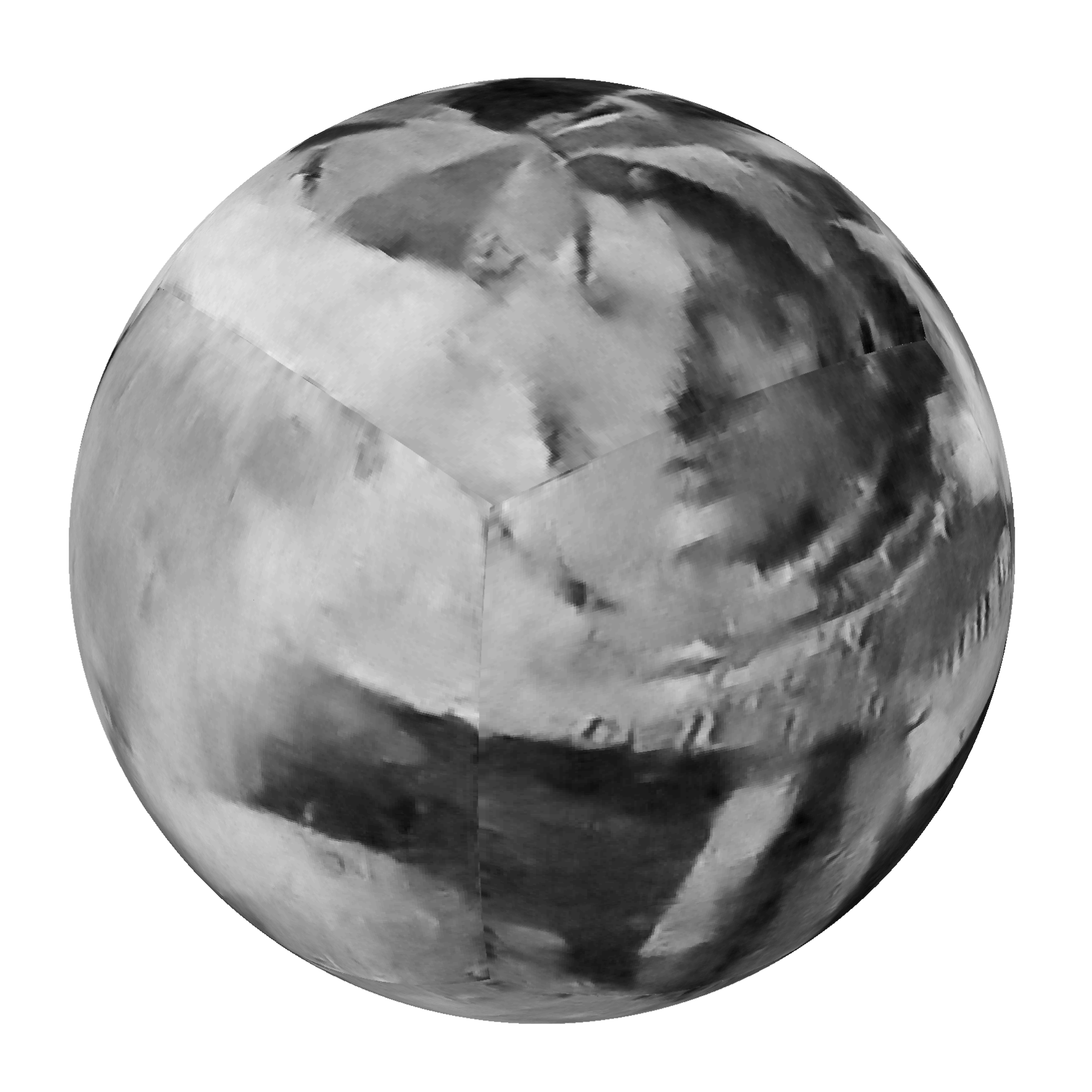}}
	\caption{Spherical Barbara with Gaussian noise and the corresonding denoised results.}
	\label{reg_noise}
\end{figure}

\section{Conclusion and further remarks}

In this work, we investigate the approximation ability of neural networks activated by SignReLU. We remark that (a) SignReLU is able to produce rational functions and approximate ReLU efficiently, (b) an improved approximation power is achieved by using SignReLU neural networks, compared with those activated by ReLU and rational activation functions and we extend the measure of error from $L_\infty$ to $L_{p,w}$, $p\in[1,\infty)$, and (c) approximation results on Korobov space and tendered BV functions solves the curse of dimensionality. We would like to mention that most of our techniques are also suitable for activation functions that are able to realize polynomial/rational functions. Several experiments are conducted and show the ability of SignReLU in deep learning tasks.
	
There are many research directions for further work. Based on \propref{exp}, we are able to discuss the relationship between neural networks and reproducing kernel Hilbert spaces. Since ReLU neural networks can efficiently produce piecewise linear functions and SignReLU neural networks are able to realize rational functions, it could be more powerful to combine these two activations to enhance the learning ability in applications.

\section*{Acknowledgement}
The second and last authors are supported partially by the Laboratory for AI-Powered Financial Technologies, the Research Grants Council of Hong Kong [Projects \# C1013-21GF, \#11306220 and \#11308121], the Germany/Hong Kong Joint Research Scheme [Project No. G-CityU101/20], the CityU Strategic Interdisciplinary Research Grant [Project No. 7020010], National Science Foundation of China [Project No. 12061160462],
and Hong Kong Institute for Data Science.

\newpage
\begin{appendix}
	\section{Basics properties of SignReLU nets}
	In the following, for simplicity, we only use $\L$, $\W$, and $\N$ to denote the depth, width, and number of weights of a given neural network $\Phi$, if it is clear from the context.
	
        Before proving the main results, we introduce several basic operations used frequently in our proofs.
        \begin{lem}[Composition]\label{compos}
            Let $d,d_1,d_2\in \NN_+$ and $M > 0$. Assume that there are two SignReLU neural networks $\Phi_1:[-M,M]^d \rightarrow [-M,M]^{d_1}$ with depth $\L_1$, width $\W_1$ and number of weighs $\N_1$ and $\Phi_2 : \RR^{d_1} \rightarrow \RR^{d_2} $ with depth $\L_2$, width $\W_2$ and number of weighs $\N_2$. Then there exists a SignReLU neural network $\Phi$ with depth $\L = \L_1 +\L_2+1$, width $\W \leq \max\{\W_1,\W_2\}$ and number of weights $\N \leq \N_1 + \N_2$ such that $\Phi(\b x) = \Phi_2 \circ \Phi_1(\b x)$ for all $x\in [-M, M]^d$.
        \end{lem}
        \begin{proof}
        We denote $\b 1_m \in \RR^m$ as the vector with all elements to be $1$.
            By definition \eqref{defnn}, we assume that
            \begin{align*}
                \Phi_1 &= \A^1_{\L_1+1} \circ \rho \circ \A^1_{\L_{1}} \circ  \cdots  \circ \rho \circ \A^1_1, \\
                \Phi_2 &= \A^2_{\L_2+1} \circ \rho \circ \A^2_{\L_{2}} \circ   \cdots \circ \rho \circ \A^2_1,
            \end{align*}
            where $\A^i_j = \b A^i_j \b y + \b b^i_j$ are affine transforms for some matrices $\b A^i_j \in \RR^{d_j^i \times d_{j-1}^i }$ and $\b b^i_j \in \RR^{d^i_j}$, $i=1,2$, $j \leq \max \{\L_1, \L_2 \}$. Notice that $d^2_0 = d^1_{\L_1+1} = d_1$.

            Define the following neural network $\Phi$
            \begin{align}\label{compos2}
            \begin{aligned}
                \Phi
                &= \A^2_{\L_2+1} \circ \rho \circ \A^2_{\L_{2}} \circ   \cdots \circ \rho \circ \tilde{\A}^2_1 \circ \rho \circ \tilde{\A}^1_{\L_{1}+1} \circ  \cdots  \circ \rho \circ \A^1_1 \\
            \end{aligned}
            \end{align}
            where $\tilde{\A}^1_{\L_1+1} (\b y) := \b A^1_{\L_1+1}\b y + \b b^1_{\L_1+1}+M \b 1_{d_1}$ and $\tilde{\A}^2_{1} (\b y) := \b A^2_1 \b y+ \b b^2_1-M \b A^2_1 \b 1_{d_1}$. Obviously, when $\A^1_{\L_1+1} (\b y) \in [-M,M]^{d_1}$, then $\rho \left(\tilde{\A}^1_{\L_1+1} (\b y)\right) = \tilde{\A}^1_{\L_1+1} (\b y) $. Hence, for any $\b y$ that satisfies $\A^1_{\L_1+1} (\b y) \in [-M,M]^{d_1}$, we have
            \begin{align}\label{compos3}
            \begin{aligned}
                &\tilde{\A}^2_1 \circ\rho\circ \tilde{\A}^1_{\L_1+1} (\b y) \\
                &= \tilde{\A}^2_1 \circ \tilde{\A}^1_{\L_1+1} (\b y) \\
                &= \A^2_1 \circ \A^1_{\L_1+1} (\b y).
            \end{aligned}
            \end{align}
            Combining \eqref{compos2}, \eqref{compos3} and the assumption that $\Phi_1:[-M,M]^d \rightarrow [-M,M]^{d_1}$, we conclude that $\Phi = \Phi_2 \circ \Phi_1$ and it is a SignReLU neural network with $\L= \L_1 + \L_2 +1$, $\W \leq \max \{\W_1, \W_2 \}$ and $\N \leq \N_1 + \N_2$.
        \end{proof}

        \begin{lem}[Summation]\label{parallel}
            Let $d_1, d_2,m\in \NN_+$. Denote $\Phi = \sum_{i=1}^m \alpha_i \Phi_i$, where $\Phi_i:\RR^{d_1} \rightarrow \RR^{d_2}$ are SignReLU neural networks with depth $\L$, width $\W_i$ and number of weights $\N_i$. Then $\Phi:\RR^{d_1} \rightarrow \RR^{d_2}$ can be represented as a SignReLU neural network with depth $\L$, width $ \sum_{i=1}^m\W_i $, and number of weights $\sum_{i=1}^m\N_i$.
        \end{lem}
        \begin{proof}
            By definition \eqref{defnn}, we denote
            \begin{align*}
                \Phi_i &= \A^i_{\L+1} \circ \rho \circ \A^i_{\L} \circ  \cdots  \circ \rho \circ \A^i_1, \quad i = 1,\dots,m,
            \end{align*}
            where $\A^i_j(\b y) = \b A^i_j \b y + \b b^i_j$ are affine transforms with matrices $\b A^i_j \in \RR^{d_j^i \times d_{j-1}^i }$ and $\b b^i_j \in \RR^{d^i_j}$. Define
            \begin{align}
                \Phi := \A_{\L+1} \circ \rho \circ \A_{\L} \circ  \cdots  \circ \rho \circ \A_1 ,
            \end{align}
            where $\A_j (\b y) = \b A_j \b y + \b b_j $ with
            \begin{align}
            \begin{aligned}
                   \b A_1 =
            \begin{pmatrix}
                 \b A^1_1  \\
                 \vdots \\
                 \b A^m_1
            \end{pmatrix},\quad
                \b b_1 =
            \begin{pmatrix}
                 \b b^1_1  \\
                 \vdots \\
                 \b b^m_1
            \end{pmatrix}
            ,
            \end{aligned}
            \end{align}
        	for $j=1$,
            \begin{align}
            \begin{aligned}
                   \b A_j =
            \begin{pmatrix}
                 \b A^1_j & \cdots & 0 \\
                 \vdots   & \ddots & \vdots \\
                 0        & \cdots & \b A^m_j
            \end{pmatrix},\quad
                \b b_j =
            \begin{pmatrix}
                 \b b^1_j  \\
                 \vdots \\
                 \b b^m_j
            \end{pmatrix}
            ,
            \end{aligned}
            \end{align}
        	for $2\leq j \leq \L$, and
            \begin{align}
            \begin{aligned}
                   \b A_{\L+1} =
            \begin{pmatrix}
                 \alpha_1 \b A_{\L+1}^1  &
                 \cdots &
                 \alpha_m \b A_{\L+1}^m
            \end{pmatrix},\quad
                \b b_{\L+1} = \sum_{i=1}^{m}\alpha_i\b b^i_{\L+1}
            ,
            \end{aligned}
            \end{align}
        	for $j = \L + 1$.

            It is easy to see that $\Phi(\b x) = \sum_{i=1}^m \alpha_i \Phi_i(\b x)$ for any $\b x \in \RR^{d_1}$ and it is a SignReLU neural network of depth $\L$, width $ \sum_{i=1}^m\W_i $, and number of weights $\sum_{i=1}^m\N_i $.
        \end{proof}

        \begin{lem}[Concatenation]\label{concat}
            Let $m,d,d_i\in \NN_+$, $i=1,\dots,m$. Denote $\Phi(\b x) = \left(\Phi_1(\b x)^\top, \Phi_2(\b x)^\top  , \dots,\Phi_m(\b x)^\top\right)^\top$, where $\b x \in \RR^d$ and $\Phi_i:\RR^{d} \rightarrow \RR^{d_i}$ are SignReLU neural networks with depth $\L$, width $\W_i$ and number of weights $\N_i$. Then $\Phi:\RR^{d} \rightarrow \RR^{\sum_{i=1}^{m}d_i}$ can be represented as a SignReLU neural network with depth $\L$, width $\sum_{i=1}^m\W_i$, and number of weights $\sum_{i=1}^m\N_i$.
        \end{lem}
        \begin{proof}
            By definition \eqref{defnn}, we denote
            \begin{align*}
                \Phi_i &= \A^i_{\L+1} \circ \rho \circ \A^i_{\L} \circ  \cdots  \circ \rho \circ \A^i_1, \quad i = 1,\dots,m,
            \end{align*}
            where $\A^i_j(\b y) = \b A^i_j \b y + \b b^i_j$ are affine transforms with matrices $\b A^i_j \in \RR^{d_j^i \times d_{j-1}^i }$ and $\b b^i_j \in \RR^{d^i_j}$. Define
            \begin{align}
                \Phi := \A_{\L+1} \circ \rho \circ \A_{\L} \circ  \cdots  \circ \rho \circ \A_1 ,
            \end{align}
            where $\A_j (\b y) = \b A_j \b y + \b b_j $ with
            \begin{align}
            \begin{aligned}
                   \b A_1 =
            \begin{pmatrix}
                 \b A^1_1  \\
                 \vdots \\
                 \b A^m_1
            \end{pmatrix},\quad
                \b b_1 =
            \begin{pmatrix}
                 \b b^1_1  \\
                 \vdots \\
                 \b b^m_1
            \end{pmatrix}
            ,
            \end{aligned}
            \end{align}
            and when $2\leq j \leq \L+1$,
            \begin{align}
            \begin{aligned}
                   \b A_j =
            \begin{pmatrix}
                 \b A^1_j & \cdots & 0 \\
                 \vdots   & \ddots & \vdots \\
                 0        & \cdots & \b A^m_j
            \end{pmatrix},\quad
                \b b_j =
            \begin{pmatrix}
                 \b b^1_j  \\
                 \vdots \\
                 \b b^m_j
            \end{pmatrix}
            .
            \end{aligned}
            \end{align}
            It is easy to see that $\Phi(\b x) = \left(\Phi_1(\b x)^\top, \Phi_2(\b x)^\top  , \dots,\Phi_m(\b x)^\top\right)^\top$ for any $\b x \in \RR^d$. Since compared with $\Phi_i$, no other nonzero elements are introduced in $\b A_j \in \RR^{d_j \times d_{j-1}}$ and $d_j \leq \sum_{i=1}^m d^i_j$ , the SignReLU neural network $\Phi$ is of depth $\L$, width $\sum_{i=1}^m\W_i$, and number of weights $\sum_{i=1}^m\N_i$.
        \end{proof}

    Since in the following, we will frequently use \lemref{compos}, \lemref{parallel} and \lemref{concat}, when we handle the composition/summation/concatenation between neural networks, we will simply define the resulting neural network as the composition/summation/concatenation and recompute the depth, width, and number of weights accordingly.

	We first give the proof of \lemref{times}, which follows some ideas of \cite{shen2021deep}[Lemma 17].
	
	\begin{proof}[{\bf Proof of \lemref{times}}]
            Since the product gate $x\cdot y = \f{1}{2}(x+y)^2 - x^2 - y^2$ relies on squaring function, we first construct a neural network that can realize $x^2$.

		For any $x \in [-1,1]$, we have $-x-1 \leq 0$, $-x-2 \leq 0$ and thereby
		\begin{align}\label{lem1eq1}
            \begin{aligned}
                & 1-12\rho(-x-1)+12\rho(-x-2) \\ &=1-12\f{-x-1}{x+2}+12\f{-x-2}{x+3} \\
			 &= 1 - \f{12}{(x+2)(x+3)} .
            \end{aligned}
		\end{align}
		
		Hence, combining \eqref{lem1eq1} and \eqref{rho} with the fact that $0<(x+2)(x+3)\leq 12$ for $x\in [-1,1]$, we can get
		\begin{align}\label{lem1eq2}
            \begin{aligned}
                & 12 \rho \left(1-12\rho(-x-1)+12\rho(-x-2) \right) \\
			 &= \left(1-\f{12}{(x+2)(x+3)}\right)(x+2)(x+3) \\
			 &=x^2 + 5x - 6.
            \end{aligned}
		\end{align}
		
		As $\rho(x+1)=x+1$ and $6-5x \geq 0$ for $x \in [-1,1]$, we have
            \begin{align}\label{lem1eq3}
                \begin{aligned}
                    11\rho\left( \f{11-5\rho(x+1)}{11} \right) = 11\rho \left( \f{6-5x}{11} \right) = 6-5x.
                \end{aligned}
            \end{align}

		Hence, adding \eqref{lem1eq2} and \eqref{lem1eq3}, we find the following SignReLU neural network $\Phi_1$ that realize $x^2$
		\begin{align}\label{square}
		\Phi_1(x):=12\rho\Big(1-12\rho (-x-1) +12\rho(-x-2) \Big) + 11\rho(\f{11-5\rho(x+1)}{11})
		=x^2.
		\end{align}
		The above formula corresponds to a realization of $x^2$ by a SignReLU neural network with $\L = 2$, $\W=3$, and $\N = 13$.	
		Then for $x,y \in [-M,M]$, substituting \eqref{square} into the following expression, the neural network
		\begin{align}\label{lem1phi}
            \begin{aligned}
                \Phi(x,y) &:= 2M^2  \Phi_1\left( \rho\left(\f{1}{2M}(1,1)(x,y)^\top + 1\right) - 1\right) - 2M^2\Phi_1\left( \rho\left(\f{1}{2M}(1,0)(x,y)^\top + 1\right) - 1\right) \\
                &\quad - 2M^2 \Phi_1\left( \rho\left(\f{1}{2M}(0,1)(x,y)^\top + 1\right) - 1\right) \\
                &= 2M^2 \Big( \Phi_1(\f{x+y}{2M}) - \Phi_1(\f{x}{2M})- \Phi_1(\f{y}{2M}) \Big)\\
                &=xy,
            \end{aligned}
		\end{align}
		is a realization of the product function $x\cdot y$ on $[-M, M]\times [-M,M]$ by a SignReLU neural network. Combining \lemref{compos}, \lemref{parallel} and \eqref{square}, the network $\Phi$ is of $\L = 4$, $\W \leq 9$ and $\N \leq 63$. This proves the statement in (i).
		
		To see (ii), noticing that $1-x/a<0$ for $x>a>0$,  we have that
		\begin{align*}
		\rho\big( 1-\f{x}{a} \big)
		= \f{1-\f{x}{a}}{1 - (1-\f{x}{a}) }= -1 + \f{a}{x}.
		\end{align*}
		Thus
            \begin{align}\label{1x}
                \f{1}{a} \rho \left(1-\f{1}{a}x \right)+\f{1}{a} =\f{1}{x}
            \end{align}
		is able to be realized by a neural network with $\L = 1$, width $\W = 1$ and $\N = 4$.

        Based on \eqref{1x}, we define $\Phi_2(x,y)$ as
            \begin{align}\label{phi2}
                \begin{aligned}
                    \Phi_2(x,y) :=
                \begin{pmatrix}
                 \f{1}{a}  & 0 \\
                 0 &  1
                \end{pmatrix}
                \rho
                \left(
                \begin{pmatrix}
                    -\f{1}{a} & 0 \\
                    0 & 1
                \end{pmatrix}
                \begin{pmatrix}
                    x\\
                    y
                \end{pmatrix}
                +
                \begin{pmatrix}
                    1 \\
                    M
                \end{pmatrix}
                \right)
                +
                \begin{pmatrix}
                    \f{1}{a} \\
                    -M
                \end{pmatrix}
                =
                \begin{pmatrix}
                    \f{1}{x}\\
                    y
                \end{pmatrix}
                ,\quad \forall (x,y)\in [a,M]\times[-M,M].
                \end{aligned}
            \end{align}
		Then, by \eqref{lem1phi} and \lemref{compos}, the neural network $\Phi_3:=\Phi \circ \Phi_2$ is of $\L = 6$, $\W = 9$ and $\N \leq 71$ and satisfies $\Phi_3(x,y) = \f{y}{x}$, for any $x \in [a, M]$ and $y \in [-M, M]$.
	\end{proof}
	
	In \eqref{lem1phi}, we can see all the quantities like the one in the second equality is able to be represented as a fully connected neural network with depth, width and number of weights only increase in terms of the input dimension. In the following, for simplicity, we will use this observation without explanation.
	
	\begin{proof}[{\bf Proof of \lemref{ra}}]
		
		Without loss of generality, we consider the construction on domain $[-1,1]$. We start by proving that any polynomial on $[-1,1]$ of degree at most $n$ can be achieved by a SignReLU neural network. Then rational functions are shown by combining polynomial results and product \& division gates in \lemref{times}.

		Let us first consider how to realize a neural network $\phi$ that is fed $(x, w, y,z )^\top \in [-1,1]\times[-M,M]^3$ and output $(x, axw+by, w, z+cw  )^\top \in [-1,1]\times[-M,M]^3$ for some given constants $a,b,c$. By \lemref{times}, there exists a SignReLU neural network $\psi_1(x,y)$ such that $\psi_1(x,y) = xy$ for any $x,y\in[-M,M]$. The following neural network $\psi_2$ obviously realizes an identity map thanks to the linear part of SignReLU \eqref{SignReLU}
        \begin{align}\label{ra11}
            \psi_2(y):= \rho(y+M)-M  = y, \forall y\in[-M,M].
        \end{align}

    	Notice that $\psi_2 \circ \psi_2 = \psi_2$ on $[-M,M]$. We will abuse $\psi_2$ to be any neural network that may have arbitrary depth and realizes the indentity map on $[-M, M]$, for the sake of the conditions needed in \lemref{parallel} and \lemref{concat}. If $\psi_2$ has depth $L$, then the number of weights is no more than $4L$.
    	
        Now we are able to see the construction
        \begin{align}\label{ra22}
            \begin{aligned}
                \phi(x,w,y,z) &= \left( \psi_2(x), a\psi_1(x,w)+ b \psi_2(y), \psi_2(w), \psi_2(z) + c\psi_2(w) \right)^\top \\
                &=(x, axw+by, w, z+cw)^\top
            \end{aligned}
        \end{align}
        satisfy our needs and by \lemref{times}, \lemref{parallel} and \lemref{concat}, it is easy to see that it is a SignReLU neural network with $\L $, $\W $ and $\N $ are some constants.

        Assume that the polynomial $P_n$ has the expansion $P_n(x) = \sum_{i=0}^n d_i p_i(x)$, where $p_i$ are Legendre polynomials. Recall that Legendre polynomials $p_j(x)$ satisfies a three-term recurrence relationship
		\begin{align}\label{three-term}
		p_{j+1} (x) = \f{2j+1}{j+1} x p_{j}(x)  - \f{j}{j+1} p_{j-1}(x),
		\end{align}
		where $p_0(x) \equiv 1$, $p_1(x) = x$.

        Define
        \begin{align}
            \Phi_j(x) := \underbrace{\phi \circ \phi \circ \cdots \circ\phi(x)}_{j-1}\circ \Phi_1( x),
        \end{align}
        and $\Phi_1(x):= \left(x, x,1,d_0  \right)^\top$.
        Then according to \eqref{ra22}, \eqref{three-term}, letting $a = \f{2+1}{1+1}$, $b = -\f{1}{1+1}$ and $c = d_1$, we have
        \begin{align}
            \Phi_2(x) = \phi \circ \Phi_1(x) = \left( x, a \psi_1(x,x) + b \psi_2(1), \psi_2(x), \psi_2(d_0) + d_1\psi_2(x) \right)^T =\left( x,p_2(x), p_1(x), d_0+d_1x  \right)^T.
        \end{align}
        Assume that the following equality holds
        \begin{align}
            \begin{aligned}
                \Phi_j(x) = \left( x, p_j(x),p_{j-1}(x), \sum_{i=0}^{j-1}d_{i}p_{i}(x)  \right)^\top.
            \end{aligned}
        \end{align}
        Then we have
        \begin{align}\label{ra33}
            \begin{aligned}
                \Phi_{j+1}(x) &= \phi \circ \Phi_{j}(x) \\
                &=  \left( x, a xp_j(x)+ b p_{j-1}(x),p_{j}(x), \sum_{i=0}^{j-1}d_{i}p_{i}(x) + c p_j(x)  \right)^\top.
            \end{aligned}
        \end{align}
        If we choose $a = \f{2j+1}{j+1}$, $b = -\f{j}{j+1}$ and $c = d_j$ for \eqref{ra33}, any polynomial $P_n(x) = \sum_{i=0}^n d_i p_i(x)$ can be realized by the following neural network
        \begin{align}\label{rapoly}
            \Phi(x):=
            \begin{pmatrix}
                0&0&0&1
            \end{pmatrix}
            \rho \left( \Phi_{n+1}(x) + M \right) - M = P_n(x), \quad \forall x \in [-1, 1]
        \end{align}
        where we choose $M = \sup_{x\in [-1,1],j = 0,1,\dots,n}  \{|x|, |p_j(x)|, |P_j(x)| \}$. Combining \eqref{ra22}, \eqref{ra33}, \eqref{rapoly} and \lemref{compos}, $\Phi$ is a SignReLU neural network with $\L = O(n)$, $\W = O(1)$ and $\N= O(n)$.

        Let $R(x):= p(x)/q(x)$ where $p(x)$ and $q(x)$ are polynomials with degrees to be $n$, $m$, respectively. Then there exist SignReLU neural networks $\Phi_{n}(x) = p(x)$ and $\Phi_{m}(x) = q(x)$. If $m<n$, then use a similar idea to combine \eqref{ra11} and \eqref{rapoly}, $\Phi_{m}(x)$ can be easily extended to a SignReLU neural network with the same depth as $\Phi_n(x)$. Hence, combining \lemref{times}, \lemref{concat} with the polynomial result, $R(x)$ can be realized by a SignReLU neural network $\psi_1\left(\Phi_n(x), \Phi_m(x)\right)$ with $\L = O(\max\{n,m\})$, $\W = O(1)$ and $\N = O(\max\{n,m\})$.
	\end{proof}
	
	The following lemma shows how SignReLU nets can approximate ReLU.
	
	\begin{proof}[{\bf Proof of \lemref{relu}}]
		Define $\Phi(x):= \f{\rho(nx)}{n}$, it is easy to check that for any $x \in \RR$, we have $|\sigma(x)-\Phi(x)|\leq \f{1}{n}$. Then the first statement follows by taking $n \geq \f{1}{\varepsilon}$.
		
		To prove (ii), we define $\Phi_m(x):= \rho \circ \rho \circ \cdots \circ \rho (x)$ with $m$ compositions which is equal to $x$ when $x$ is nonnegative and $\f{x}{1-mx}$ otherwise. Then $\Phi_m(x) = \sigma(x)$ for $x\geq 0$ and $|\Phi_m(x) - \sigma(x)| = \f{-x}{1-mx} \leq \f{1}{m}$ for $x<0$. Hence the statement in (ii) follows.
		
		For the last statement, we choose the network $\Phi(x): = \Phi_n \circ \Phi_n  \circ \cdots \circ \Phi_n (x) = \f{x}{1-mnx}$ with $m$ compositions. It is easy to see the conclusion.
	\end{proof}

    \begin{proof}[{\bf Proof of \propref{exp}}]
            The idea of the proof is to construct a SignReLU neural network $\phi_0$ that approximates $e^{-|x|}$, and then combine it with previous results for the product gate (\lemref{times}) and rational functions (\lemref{ra}) to obtain approximation rates for target functions.

            {\bf Step 1: Constructing SignReLU net $\phi_0$ that approximates $e^{-x}$.}
		Let $\psi(x):= \f{1}{\lambda } \left(\rho(\lambda x) + \rho(-\lambda x)
		\right)$ for $x\in \RR$ and $\lambda >1$. It is easy to see that
		\[\psi(x)=\left\{
		\begin{array}{ll}
		x-\f{x}{1+\lambda x}, & \hbox{if $x\geq 0$,} \\
		-x+\f{x}{1-\lambda x}, & \hbox{if $x<0$,}
		\end{array}
		\right.
		\]
		which implies that $0 \leq \psi(x) \leq |x|$ and $\big | \psi(x)-|x|\big | \leq
		\f{1}{\lambda }$ for any $x \in \RR$.
		Furthermore, since $1-e^{-x}\leq x$ and $e^{-\psi(x)}\leq 1$ for any $x\in\RR$, we have
		\begin{align}\label{e-1}
            \begin{aligned}
                \big|e^{-|x|}-e^{-\psi(x)} \big| =\left|e^{-\psi(x)}\right| \left|1-e^{ -\left(|x|-\psi(x) \right) }  \right|  \leq \left|1-e^{ -\left(|x|-\psi(x) \right) }  \right|
		 \leq \f{1}{\lambda }, \quad \forall x\in \RR.
            \end{aligned}
		\end{align}
		On the other hand, by a classical result on rational approximation to $e^{-x}$,
		$x\geq 0$ (see, e.g.\cite{lorentz1996constructive}), for any $n\in\NN$, there
		exists a polynomial $q(x)$ of degree at most $n $ such that
		\begin{align}\label{re27}
		\Big | \f{ 1 }{ q(x) } - e^{-x} \Big | \leq \sqrt{2} 3^{-n}, \ \forall x
		\geq 0.
		\end{align}
		Combining \eqref{e-1} and \eqref{re27}, we have
		\begin{align}\label{expphi0}
            \begin{aligned}
                \quad \Big | \f{ 1 }{ q(\psi(x)) } - e^{-|x|} \Big |
		&\leq  \Big|\f{ 1 }{ q(\psi(x)) } - e^{-\psi(x)} \Big| + \Big | e^{-\psi(x)}
		- e^{-|x|} \Big | \\
		&\leq \sqrt{2} 3^{-n} + \f{1}{\lambda }, \ \forall x\in \RR.
            \end{aligned}
		\end{align}
		Finally, taking $\lambda := 3^{n}$, by \lemref{ra}, we can construct a SignReLU
		neural network $\phi_0(x)$ with depth $\L = O(n)$, width $\W = O(1)$ and number of weights
		$\N = O(n)$ such that $\phi_0(x)= Q \circ \rho \circ \psi (x) =  \f{1}{q(\psi(x))}$ and $\left|\phi_0(x) - e^{-|x|}\right|\leq 3^{1-n}$ where $Q(x)$ is a SignReLU neural network and satisfies $Q(x) = \f{1}{q(x)}$ and $\rho \circ \psi(x) = \psi(x)$ since $\psi(x) \geq 0$, $\forall x \in \RR$.

		{\bf Step 2: Approximating $e^{-\|\b x\|_1}$.} Let $\psi_d(\b x) = \sum_{j=1}^{d} \psi(x_j)$ and
		$\Phi_d(\b x) = Q \circ \rho \circ \psi_d(\b x)$. Applying the same arguments as in \eqref{expphi0}, we have that
		\begin{align*}
		\left|e^{-\|\b x\|_1}- \Phi_d(\b x)\right|
		&\leq \sqrt 2 3^{-n}+\f d{\lambda}.
		\end{align*}
		Then by taking $\lambda=3^n d$, we can get the desired result in $d$-dimensional
		case.	

            {\bf Step 3: Approximating $e^{-\|\b x\|_2^2}$.} By \lemref{times}, there exist SignReLU neural networks $\phi_i(x) = x^2 $, $i=1,\dots,d$ such that $\|\b x \|^2_2 = \sum_{i=1}^d \phi_i(x_i)$. Define $\Phi_d(\b x):= Q \circ \rho \left(\sum_{i=1}^d \phi_i(x_i) \right)$. Then combining \eqref{re27}, we have
            \begin{align}\label{exp33}
                \left|\Phi_d(\b x) - e^{-\|\b x \|_2^2}\right| \leq  \sqrt{2}3^{-n}.
            \end{align}
            Letting $n = \ln (\varepsilon^{-1})+1$ and combining \eqref{exp33} with \lemref{parallel} and \lemref{compos}, we can get the desired result.
	\end{proof}

        \begin{proof}[{\bf Proof of \thmref{relunets}}]
            Given any fixed rational activation function $R(x)$ \cite{boulle2020rational}, it can be produced by a SignReLU network with fixed size (only depends on the degree of $R(x)$, \lemref{ra}), and thus the first statement holds.


            Let $\{ \A_\ell \}_{\ell=1}^{L}$ be a collection of linear transforms. For any $\A_\ell(\b y) = \b A_\ell \b y + \b b_\ell$ for some matrix $\b A_\ell$ and vector $\b b_\ell$, we denote $a_\ell := \max \{\|\b A_\ell\|_{\infty,\infty}, \|\b b_\ell \|_\infty  \}$, where $\|\b A \|_{\infty,\infty}:= \max_{ij} \{\b A_{ij} \}$. Without loss of generality, we assume that $a_\ell \leq 1$ for all $\ell$.

            Define a ReLU neural network $f^{(L)}_\sigma $ with $\L = L-1$, $\W = W$ and $\N = N$ as
            \begin{align}\label{relunets1}
            \begin{aligned}
                f_\sigma^{(1)}&:= \sigma  \circ \A_1 ,\\
                f_\sigma^{(\ell+1)}&:=\sigma  \circ \A_{\ell+1}   \circ f_\sigma^{(\ell)},\\
                f_\sigma^{(L)} &:= \A_{L} \circ f_\sigma^{(L-1)},
            \end{aligned}
            \end{align}
            where $\A_L: \RR^{d_{L-1}} \rightarrow \RR$,
            and a SignReLU neural network $f^{(L)}_\rho$ activated by $\rho$ as
            \begin{align}\label{relunets2}
            \begin{aligned}
                f_\rho^{(1)}&:= \rho \circ \f{1}{\delta}\A_1 ,\\
                f_\rho^{(\ell+1)}&:=\rho \circ \A_{\ell+1}   \circ f_\rho^{(\ell)},\\
                f_\rho^{(L)} &:= \delta\A_{L} \circ f_\rho^{(L-1)}.
            \end{aligned}
            \end{align}
            for some constant $\delta>0$. In the following, we denote $f^{(\ell)}_\sigma(\b x)_j$ the $j$-th element of $f^{(\ell)}_\sigma(\b x)$. Since $a_\ell \leq 1$ and $d_{L-1} \leq W$, we have
            \begin{align}\label{relunet_last}
                \begin{aligned}
                    &\left| f^{(L)}_\sigma(\b x) - f^{(L)}_\rho (\b x) \right| \\
                    &\leq \sum_{j=1}^{d_{L-1}} (\b A_L)_{1j} \left | f^{(L-1)}_\sigma(\b x)_j - \delta f^{(L-1)}_\rho(\b x)_j \right | \\
                    & \leq W\max_{j=1,\dots,d_{L-1}} \left | f^{(L-1)}_\sigma(\b x)_j - \delta f^{(L-1)}_\rho(\b x)_j \right |.
                \end{aligned}
            \end{align}
            Notice that
            \begin{align}\label{relunets3}
                \begin{aligned}
                     &\left | f^{(\ell)}_\sigma(\b x)_j - \delta f^{(\ell)}_\rho(\b x)_j \right | \\
                     &=  \left| \sigma\left( \sum_{k = 1}^{d_{\ell-1}}   (\b A_{\ell})_{jk} f^{(\ell-1)}_\sigma(\b x)_k \right ) - \delta \rho\left(\sum_{k = 1}^{d_{\ell-1}}   (\b A_{\ell})_{jk} f^{(\ell-1)}_\rho (\b x)_k \right ) \right | .
                \end{aligned}
            \end{align}
            If $L = 1$, then by \eqref{relunets3}, we can get
            \begin{align}\label{relunets4}
                \begin{aligned}
                     &\left | f^{(1)}_\sigma(\b x)_j - \delta f^{(1)}_\rho(\b x)_j \right | \\
                     &=  \left| \sigma\left( \sum_{k = 1}^{d_{0}}   (\b A_{1})_{jk}  x_k \right ) - \delta \rho\left(\sum_{k = 1}^{d_{0}}  \f{1}{\delta} \cdot (\b A_{1})_{jk}  x_k \right ) \right | \\
                     &\leq \delta.
                \end{aligned}
            \end{align}
            where in the last inequality, we used $|\sigma(x) - \delta \rho(x/\delta)| \leq \delta$, $\forall x \in \RR$.


            Assume that for $L = \ell-1$, the following inequality holds
            \begin{align}\label{relunets5}
                \begin{aligned}
                     \max_{j = 1,\dots, d_{\ell-1}}\left | f^{(\ell-1)}_\sigma(\b x)_j - \delta f^{(\ell-1)}_\rho(\b x)_j \right | \leq C_{\ell-1}\delta.
                \end{aligned}
            \end{align}
            for some constant $C_{\ell-1}>0$.
            Combining \eqref{relunets3} with \eqref{relunets5} and $a_\ell \leq 1$, we obtain
            \begin{align}\label{relunets6}
                \begin{aligned}
                     &\left | f^{(\ell)}_\sigma(\b x)_j - \delta f^{(\ell)}_\rho(\b x)_j \right | \\
                     &=  \left| \sigma\left( \sum_{k = 1}^{d_{\ell-1}}   (\b A_{\ell})_{jk} f^{(\ell-1)}_\sigma(\b x)_k \right ) - \delta \rho\left(\sum_{k = 1}^{d_{\ell-1}}   (\b A_{\ell})_{jk} f^{(\ell-1)}_\rho (\b x)_k \right ) \right | \\
                     &\leq
                     \left| \sigma\left( \sum_{k = 1}^{d_{\ell-1}}   (\b A_{\ell})_{jk} f^{(\ell-1)}_\sigma(\b x)_k \right ) - \sigma\left(\delta \sum_{k = 1}^{d_{\ell-1}}   (\b A_{\ell})_{jk} f^{(\ell-1)}_\rho (\b x)_k \right ) \right |  \\
                     &\quad +
                     \left| \sigma\left(\delta \sum_{k = 1}^{d_{\ell-1}}   (\b A_{\ell})_{jk} f^{(\ell-1)}_\rho(\b x)_k \right ) - \delta \rho\left(\sum_{k = 1}^{d_{\ell-1}}   (\b A_{\ell})_{jk} f^{(\ell-1)}_\rho (\b x)_k \right ) \right | \\
                     &\leq WC_{\ell-1} \delta + \delta := C_{\ell} \delta
                     ,
                \end{aligned}
            \end{align}
            where the last inequality follows from $|\sigma(x)-\sigma(y)|\leq |x-y|$ for any $x,y$ and $\left| \sigma(\delta x)-\delta \rho(x) \right| \leq \delta$ for any $x$.
            Hence, combining \eqref{relunets6}, \eqref{relunet_last} and choosing $\delta$ which satisfies $W C_{L-1} \delta \leq  \varepsilon$, we conclude $\left| f^{(L)}_\sigma(\b x) - f^{(L)}_\rho (\b x) \right| \leq \varepsilon$ for any $\b x\in\RR^d$. The proof is completed.

        \end{proof}

	\section{Proof of \thmref{sobolev}, \thmref{piecewise smooth}, and \thmref{koborov}}
	We first derive a lemma of orthogonal expansions, which will play a key role in sequential proofs.
	\begin{lem}\label{decomposition}
		Let $r, d\in\NN_+$, $p\geq 1$ and $\b n \in \NN^d$.
		\begin{enumerate}
			\item For any $f\in W^r_p([-1,1]^d, w)$ and $N\in \NN$, there exists $c_{\b n}\in \RR$, $\|\b n\|_\infty \leq N$, such that
			\begin{equation}\label{decomposition_1}
			\|f-\sum_{\|\b n\|_\infty \leq N} c_{\b n} {\b x}^{\b n}\|_{p,w}\leq C N^{-r}\|f\|_{W^r_p}.
			\end{equation}
			\item For any $f\in K^r_p([-1,1]^d,w)$ and $N\in \NN$, there exists $c_{\b n}\in \RR$, $\|\b n\|_{\pi} \leq N$, such that
			\begin{equation}\label{decomposition_2}\|f-\sum_{\|\b n\|_{\pi} \leq N} c_{\b n} {\b x}^{\b n}\|_{p,w}\leq C N^{-r}(\log N)^{(d-1)(r+1)} \|f\|_{K^r_p},
			\end{equation}
			where $\|\b n\|_{\pi}=\prod_{j=1}^d \max\{1, n_j\}$, $\|\cdot\|_{p,w}$ is the weighted $L_p$ norm with $w$.
		\end{enumerate}
	\end{lem}
	\begin{proof}
		For $f\in  W^r_p\left([-1,1]^d,w\right)$, we define a $2\pi$-periodic function $G_f$  by
		\[G_f(\b \bt)=
		f\left(\cos \t_1,\cdots, \cos \t_d\right),  \ \text{for}\ \t_j\in [-\pi,\pi], \ j=1,2,\cdots, d.\]

            Noting that, in case $d=1$,
            \begin{align}\label{lem7_1}
                \begin{aligned}
                    \int_{-\pi}^{\pi} G_f(\theta) d \theta = \int_{-\pi}^{0} G_f(\theta) d \theta + \int_{0}^{\pi} G_f(\theta) d \theta.
                \end{aligned}
            \end{align}

            We apply the change of variable $x = \cos(\theta)$ to \eqref{lem7_1}. Note that the Lebesgue measure $\mu$ of zero $\mu(0) = 0$, $\sin(\theta) = \pm \sqrt{1-\cos^2(\theta)}$. Since $d\theta = -\f{dx}{\sin(\theta)}$, $\sin(\theta)>0, \theta \in (0,\pi]$ and $\sin(\theta)<0, \theta \in [-\pi,0)$, we get
            \begin{align}\label{lem7_2}
                \begin{aligned}
                    \int_{-\pi}^{0} G_f(\theta) d \theta = \int_{\cos(-\pi)}^{cos(0)} f(x) \f{-dx}{-\sqrt{1-\cos^2(\theta)}} =  \int_{-1}^{1} f(x) (1-x^2)^{-1/2} d x.
                \end{aligned}
            \end{align}
            Hence
            \begin{align}\label{lem7_3}
                \begin{aligned}
                    \int_{-\pi}^{\pi} G_f(\theta) d \theta  =2  \int_{-1}^{1} f(x) (1-x^2)^{-1/2} d x.
                \end{aligned}
            \end{align}

		A similar approach shows the following results for $d>1$
		\begin{equation}\label{claim1}
		\int_{[-1,1]^d} f(\b x)w(\b x)d\b x=\int_{[-\pi,\pi]^d} G_f(\b \t) d\b \t,
		\end{equation}
		here recall that $w(\b x)=2^d\prod_{j=1}^d (1-x_j^2)^{-1/2}$. We can see from \eqref{claim1} that if $f \in L_p\left([-1,1]^d,w\right)$, then $G_f \in L_p\left([-\pi,\pi]^d\right)$.

		Since $G_f$ is even, we obtain the following Fourier expansion of $G_f$ in the sense of $L_p$
		\begin{align*}
		G_f(\b \t)=&\sum_{m=1}^\infty\sum_{m-1\leq \|\b n\|_\infty< m}   \hat {G_f}(\b n)\prod_{j=1}^d \cos(n_j \t_j)
		\end{align*}
		where the Fourier coefficients
            \begin{align}
                \hat{G_f}(\b n):= \int_{[-\pi,\pi]^d} G_f(\b \t) e^{-i \b n \cdot \b \theta } d\b \t = \int_{[-\pi,\pi]^d} G_f(\b \t) \prod_{j=1}^d \cos(n_j \t_j) d\b \t.
            \end{align}
		For any $N\in \NN$, setting $\Lambda_{\ell,N}=\{\b n\in \NN^d: |n_\ell|\geq N\}$, $\ell=1,\ldots, d$, then
		\begin{align}\label{lit}
			\begin{aligned}
			\scriptsize
			&\int_{[-\pi, \pi]^d}\left|G_f(\b \t)-\sum_{\|\b n\|_\infty< N} \widehat G_f(\b n)\prod_{j=1}^d \cos(n_j \t_j)\right|^p d\b \t \\
			&\leq \sum_{\ell=1}^d\int_{[-\pi, \pi]^d}\left|\sum_{\b n\in \Lambda_{\ell,N}}n_\ell^{-r} \widehat{\p_\ell^{r}G_f}(\b n)\prod_{j=1}^d \cos(n_j \t_j)\right|^p d\b \t \\
			&\leq C\sum_{\ell=1}^d \int_{[-\pi, \pi]^d}\left|\left[\sum_{\b n \in \Lambda_{\ell,N}}|n_\ell^{-r}\widehat{\p_\ell^{r}G_f}(\b n)\prod_{j=1}^d \cos(n_j \t_j)|^2\right]^{1/2}\right|^p d\b  \t \\
			&\leq C_{d,p}N^{-rp}\|f\|^p_{W^r_p},
			\end{aligned}
		\end{align}
		where the first step we use the fact $\widehat{f^{(r)}}(n) = (in)^r \hat{f}(n)$, the second and last step follows from the Littlewood-Paley inequalities with the function \[g_1(\b \t)=\sum_{\b n\in \Lambda_{\ell,N}}n_\ell^{-r} \widehat{\p_\ell^{r}G_f}(\b n)\prod_{j=1}^d \cos(n_j \t_j)\] and $g_2(\b \t)=\sum_{\b n\in \Lambda_{\ell,N}} \widehat{\p_\ell^{r}G_f}(\b n)\prod_{j=1}^d \cos(n_j \t_j)$.
		Here recall the Littlewood-Paley inequalities
		\[\left\|\sum_{I \in \mathcal{D}} c_{I} \psi_{I}\right\|_{p} \sim \left\|\left(\sum_{I \in \mathcal{D}}\left[c_{I} \psi_{I}\right]^{2}\right)^{1/ 2} \right\|_{p},\]
		where $\mathcal D$ is an index set and $\{\psi_I\}_{I\in\mathcal D}$ is an orthogonal system, $A\sim B$ means $c_1 B\leq A\leq c_2 B$ for some positive constants $c_1,c_2$.
		By substitution $\t_j=\arccos (x_j)$ into \eqref{lit} and using \eqref{claim1}, we have
		\[\left\|f-\sum_{\|\b n\|_\infty< N} c_{\b n}{\b x}^{\b n}\right\|_{p,w} \leq C_d N^{-r}\|f\|_{W^r_p},\]
		by making \[\sum_{\|\b n\|_\infty< N} c_{\b n}{\b x}^{\b n}=\sum_{\|\b n\|_\infty< N} \widehat G_f(\b n)\prod_{j=1}^d \cos\left(n_j \arccos(x_j)\right) .\]
		

		For $f\in K^r_p\left([-1,1]^d,w\right)$ and $G_f$ as defined above, by \cite[Thm4.4.1, 4.4.2]{dung2018hyperbolic} and the fact that $G_f$ is even in each variable, we have that
		\[\|G_f-T_N(G_f)\|_p\leq N^{-r}(\log N)^{(d-1)(r+1)}\|f\|_{K_p^r},\]
		where $T_N(G_f)(\b \t)=\sum_{\|\b n\|_\pi\leq N} \widehat{G_f}(\b n) \prod_{j=1}^d\cos n_j\t_j $.
		Using the same argument above, we can have \eqref{decomposition_2}.
	\end{proof}
	
	To prove Theorem~\ref{sobolev} and ~\ref{koborov}, we also need the following lemma.
	\begin{lem}\label{mono}
		Let $n, d \in \NN_+$ and $\{\beta_i, i=1,\dots,d\}$ be a set of nonnegtive integers with $\sum_i \beta_i \leq n$. Then
		\begin{enumerate}
			\item[\rm(i)] there exists a function $\Phi$ realized by a SignReLU neural network with $\L = 5 \lceil \log_2 d \rceil -1$, $\W \leq 10d$ and $\N \leq 130 d$ such that $\Phi(\b x) = \prod_{i\leq d}x_i$ for any $\b x \in [-1,1]^d$;
			
			\item[\rm(ii)] there exists a function $\Phi$ realized by a SignReLU neural network with $\L \leq 5 \lceil \log_2 n \rceil + 5 \lceil \log_2 d \rceil $, $\W \leq 10nd$ and $\N \leq  400nd$ such that $\Phi(\b x) = \b x^{\b \beta}$, for any $ \b x \in [-1,1]^d$.
		\end{enumerate}
	\end{lem}

	\begin{proof}
        Given a SignReLU neural network $\Phi:\RR^d \rightarrow \RR^{d_1}$, we denote $\Phi(\b x)_j$, $j = 1,\dots,d_1$, the $j$-th element of $\Phi(\b x)$ for any $\b x \in \RR^d$.
		Without loss of generality, we assume that $d = 2^K$ for some integer $K$. We define the SignReLU neural network $\Phi(\b x):= \phi_K(\b x)$ iteratively by
        \begin{align}\label{mono1}
            \begin{aligned}
                \phi_1(\b x) &:= \left(\phi_0(x_1,x_2), \phi_0(x_3,x_4), \dots, \phi_0(x_{2^K-1}, x_{2^K})\right), \\
                \phi_\ell(\b x) &:= \left(\phi_0\big(\phi_{\ell-1}(\b x)_1,\phi_{\ell-1}(\b x)_2\big), \phi_0\big(\phi_{\ell-1}(\b x)_3,\phi_{\ell-1}(\b x)_4\big), \dots, \phi_0\big(\phi_{\ell-1}(\b x)_{2^{K-\ell+1}-1}, \phi_{\ell-1}(\b x)_{2^{K-\ell+1}}\big)\right),
            \end{aligned}
        \end{align}
        where $\phi_0$ is a SignReLU neural network that satisfies $\phi_0(y,z) = yz$ for any $y,z \in [-1,1]$. Iteratively, it is easy to verify the output of each $\phi_\ell$ in \eqref{mono1} satisfies
        \begin{align}\label{mono2}
            \begin{aligned}
                \phi_1(\b x) &= \left(x_1 x_2, x_3 x_4, \dots, x_{2^K-1} x_{2^K}\right), \\
                \phi_\ell(\b x) &= \left(\prod_{i=1}^{2^\ell}x_i,\prod_{i=2^\ell+1}^{2^{\ell+1}}x_i, \dots, \prod_{i=2^{K}-2^\ell+1}^{2^K}x_i  \right), \ell = 2,\dots,K-1, \\
                \phi_K(\b x) &= \prod_{i=1}^d x_i.
            \end{aligned}
        \end{align}
        We observe from \eqref{mono2} that $\Phi(\b x):= \phi_K(\b x)$ satisfies $\Phi(\b x) = \prod_{i\leq d} x_i$ for any $\b x \in [-1,1]^d$. Since each $\phi_\ell$ is the concatenation of $2^{K-\ell}$ product gate $\phi_0$, by \lemref{times} and \lemref{concat}, $\phi_{\ell}$ is of $\L = 4$, $\W \leq 9 \cdot 2^{K-\ell}$ and $\N \leq 63 \cdot 2^{K-\ell}$.
        Combining \eqref{mono1} with \lemref{compos}, the SignReLU neural network $\Phi$ is of $\L = 5 K-1 = 5\log_2 d-1$, $\W = \max_{\ell=1,\dots,K} 9 \cdot 2^{K-\ell} \leq 10d $  and $\N = 63\sum_{\ell=1}^K 2^{K-\ell} \leq 130 d $. If $2^{K-1}< d < 2^K$ for some integer $K$, we can set some $x_i \equiv 1$ for $i>d$.
		

        To prove the second statement, we first construct a SignReLU neural network $\Phi_\beta(x)$ that realizes $x^\beta$, $x \in [-1,1]$ for some integer $\beta \geq 1$. Denote $\psi(x) = \rho (\b 1 x+\b 1)-\b 1$ where $\b 1 \in \RR^\beta$ is the all-one vector. Since $ x \in [-1,1]$, $\psi(x) = (x,x,\dots,x)^T \in [-1,1]^\beta$, the SignReLU neural network $\phi_\beta(x):= \Phi_\beta \circ \psi (x)$ with $\Phi_\beta(\b x)= \prod_{i=1}^{\beta}x_i$ equals $x^\beta$ on $[-1,1]$ and by \lemref{compos}, $\phi_\beta$ is of $\L = 5\lceil \log_2 \beta \rceil $, $\W \leq 10 \beta$ and $\N \leq 130\beta $.

        Denote $\Phi_{\b \beta}(\b x):= \Phi \left(\Phi_{\beta_1}(x_1), \dots,\Phi_{\beta_d}(x_d) \right)$ where $\Phi(\b x) = \prod_{i=1}^d x_i$ for any $\b x \in [-1,1]^d$. Obviously, $ \Phi_{\b \beta}(\b x) = \b x^{\b \beta}$. Since $\beta_i \leq n$ and we can expend $\Phi_{\beta_i}(x)$ of $\L = 5\lceil \log_2 \beta_i \rceil $, $\W \leq 10\beta_i$ and $\N \leq 130\beta_i$, with \eqref{ra11} and \lemref{compos}, to a SignReLU neural network $\Tilde{\Phi}_{\beta_i}$ which is of $\L = 5\lceil \log_2 n \rceil $, $\W \leq 10n$ and $\N \leq 130n+20 \lceil \log_2 n \rceil \leq 200 n$ and satisfies $\Tilde{\Phi}_{\beta_i} (x_i) = \Phi_{\beta_i}(x_i)$ for any $x_i \in [-1,1]$. Hence, by \lemref{concat} and \lemref{compos}, the SignReLU neural network $\Tilde{\Phi}_{\b \beta}(\b x):= \Phi\left( \Tilde{\Phi}_{\beta_1} (x_1),\dots, \Tilde{\Phi}_{\beta_1} (x_1) \right)$ is of $\L \leq 5 \lceil \log_2 n \rceil + 5 \lceil \log_2 d \rceil $, $\W \leq 10nd$ and $\N \leq 130d + 200nd \leq 400nd$ and satisfies $ \Tilde{\Phi}_{\b \beta}(\b x) = \b x^{\b \beta} $, $\forall \b x \in [-1,1]^d$.


		
	\end{proof}

	\begin{proof}[{\bf Proof of \thmref{sobolev} and \thmref{koborov} \textbf{(i)}} ]
		The key idea is to employ \lemref{decomposition} to construct a SignReLU network that outputs a polynomial on $[-1,1]^d$ that approximates $f$.

	{\bf Estimation for approximating weighted Sobolev functions.}


        According to \lemref{ra}, there exist SignReLU neural networks $\phi_i(x_i)$ with $\L = CN$, $\W \leq C$ and $\N \leq CN $ for some constant $C$ such that $\phi_i(x_i) = x_i^N$, $\forall x_i \in [-1,1]$. Here for depth, width and number of weights, the constant $C$ may be different, we use a single $C$ for simplicity. We denote $\Phi_1(\b x):=\left( \phi_1(x_1),\dots,\phi_d(x_d) \right)$, which, using \lemref{concat}, is of $\L = CN$, $\W = Cd$ and $\N = CNd$.

        Since \eqref{ra33} shows intermediate layers of $\phi_i(x_i)$ output $x_i^{n_i}$, $n_i=1,\dots,N-1$. Based on $\Phi_1(\b x)$, We can add at most $d\sum_{i=1}^{N-1} i \leq N^2 d$ identity mappings \eqref{ra11} to intermediate layers of $\Phi_1(\b x)$ (those identity mappings keep all $x_i^{n_i}$ to the last output layer), and by \lemref{concat} and \lemref{compos} the resulting new network $\tilde \Phi_1(\b x)$ outputs all $x_i^{n_i}$, $i=1,\dots,d$, $n_i=1,\dots,N$, which is of $\L \leq CN$, $\W \leq Cd + N^2d$ and $\N \leq CNd + 4 N^2 d$.

        Let $\Phi_2(\b x)$ be the SignReLU neural network that takes all $x_i^{n_i}$, $i=1,\dots,d$, $n_i = 1,\dots,N$ as input and outputs $\b x^{\b n}$ for all $\|\b n \|_\infty \leq N$. Obviously, by \lemref{concat} and \lemref{mono}, $\Phi_2$ can be constructed by the concatenation of neural networks $\prod_{i=1}^{d}x_i^{n_i}$ which take $(x_1^{n_1}, \dots, x_d^{n_d})$ as inputs, for all $(n_1,\dots,n_d)\in \NN^d$. Since $\# \{\b n: \|\b n \|_{\infty}\leq N \} = (N+1)^d$, using \lemref{concat} and \lemref{mono}, the network $\Phi_2(\b x)$ is of $\L \leq 5\log_2 d$, $\W \leq 10d (N+1)^d$ and $\N \leq 130d (N+1)^d$.

        Denote $\Phi_3(\b x) = \sum_{\|\b n \|_\infty \leq N} c_{\b n} \Phi_2(\b x)_{\b n}$, where $\Phi_2(\b x)_{\b n} = \b x^{\b n}$. Then combining $\tilde \Phi_1(\b x)$, $\Phi_2(\b x)$ and $\Phi_3(\b x)$ and using \lemref{compos}, we can get a SignReLU neural network $\Phi(\b x) := \Phi_3\circ\Phi_2\circ\tilde \Phi_1(\b x) = \sum_{\|\b n \|_\infty \leq N} c_{\b n} \b x^{\b n}$ for any $\b x \in [-1,1]^d$ and is of $\L \leq CN+5\log_2 d+2$, $\W \leq Cd + N^2d + 10d(N+1)^d$ and $\N \leq 150d (N+1)^d+ CNd + 4 N^2 d$.

        According to \lemref{decomposition}, given any $f$ with $\|f\|_{W_p^r} \leq 1$, there exists a polynomial $P(\b x) = \sum_{\|\b n\|_\infty \leq N} c_{\b n} \b x^{\b n}$ such that $\|f-P \|_{p,w} \leq C N^{-r}$.
        Setting $N = C^{\f{1}{r}} \varepsilon^{-\f{1}{r}}$ and choosing $\Phi(\b x) = P(\b x)$ for any $\b x \in [-1,1]^d$, we conclude that $ \Phi$ is of $\L = O\left( \varepsilon^{-\f{1}{r}} \right)$, $\W = O\left( \varepsilon^{-\f{d}{r}} \right)$ and $\N = O\left( \varepsilon^{-\f{d}{r}} \right)$ and approximate $f$ with error $\| f- \Phi \|_{p,w}\leq \epsilon$.

        For the case of $W^r_\infty ([-1,1]^d)$, the proof is similar as Theorem 4 of \cite{boulle2020rational} and \cite{Yarosky} by noticing that \thmref{relunets} and \lemref{ra} hold.

        {\bf Estimation for approximating weighted Korobov functions.}

		
		Given any function $f$ in the Koborov space with unit norm, we use a similar idea to give the bound. Since there exixts a rational neural network that produce $x^n$ with $\L = O(\log_2 n)$ and $\N = O\left((\log_2 n)^2 \right)$ (Proposition 10, \cite{boulle2020rational}), combining \lemref{mono} \rm{(i)} and \lemref{ra}, there exists a SignReLU neural network that produce $\b x^{\b n}$, $\|\b n\|_\infty \leq N$, on $[-1,1]^d$ with $\L = O(\log_2 N)$ and $\N = O\left((\log_2 N)^2 \right)$. Notice that $\# \{\b n: \|\b n \|_{\pi}\leq N \} = O(N(\ln N)^{d-1})$.
        Hence, by \lemref{decomposition} and \lemref{parallel}, there exists a SignReLU neural network that produce $\sum_{\|\b n\|_{\pi} \leq N} c_{\b n} {\b x}^{\b n}$ with $\L = O(\log_2 N)$ and $\N = O\left(N(\log_2 N)^{d+1} \right)$ that approximates the given Koborov function $f$ with error $CN^{-r}\left( \log_2 N \right)^{(d-1)(r+1)}$.

         Choosing $N = \varepsilon^{-\f{1}{r}}\left(\log_2 (\varepsilon^{-1})\right)^{d(r+1)/r}$ and applying \lemref{mono}, we have a neural network $\Phi = \sum_{\|\b n\|_{\pi} \leq N} c_{\b n} {\b x}^{\b n}$, $\b x \in [-1,1]^d$ with $\L = O\left(\log_2 (\varepsilon^{-1})\right)$ and $\N = O\left(\varepsilon^{-\f{1}{r}}\left(\log_2 (\varepsilon^{-1})\right)^{(2dr+d+r)/r}\right)$ such that $\|f-\Phi\|_{p,w}\leq \varepsilon$.

	\end{proof}
	
	\begin{proof}[{\bf Proof of \thmref{piecewise smooth} and \thmref{koborov} \textbf{(ii)}.} ]

        A general piecewise smooth function has representation $f=f_1 + f_2 \chi_\Omega$ for some $\Omega$ and functions $f_1$, $f_2$. Thanks to the approximation results \thmref{sobolev} and \thmref{koborov} \rm{(i)}, the key step in this proof is to approximate $\chi_\Omega$ and $f_2 \chi_\Omega$.

        Given $f \in S\left( B_1(W^r_p([-1,1]^d,w)) \right)$. Then we have $f = f_1 + f_2 \chi_\Omega$, where $f_1, f_2 \in B_1\left(W^r_p([-1,1]^d,w)\right)$ and
        \begin{align}\label{45_Om}
            \begin{aligned}
                \Omega =\left \{\b x\in [-1,1]^d: h(\b x)< g(\b x) \right \} =\left \{\b x\in [-1,1]^d: \max \{g(\b x)-h(\b x),0 \} > 0  \right \},
            \end{aligned}
        \end{align}
        for some SignReLU neural network $h, g $.

        The expression \eqref{45_Om} of $\Omega$ indicates that we can rewrite $\chi_{\Omega}(\b x)$ as $\chi_{\Omega}(\b x) = \chi_{(0,\infty)} \circ \left( g-h \right)(\b x)$.

        Lemma 6.1 [Chapeter 7 \cite{lorentz1996constructive} shows that there exists a rational function $R_n(x)$, $n\geq 5$ that determined by two polynomials with degrees no more than $n$, such that
        \begin{equation}\label{45_rn}
            \begin{split}
                \left| \chi_{(0,\infty)}(x) - R_n(x) \right | \leq 2 e^{-\sqrt{n}},\quad &\forall  x \in [-1,-e^{-\sqrt{n}}] \cup [e^{-\sqrt{n}},1] \\
                0 \leq R_n(x)\leq 1, \quad &\forall x \in [-e^{-\sqrt{n}} ,e^{-\sqrt{n}}].
            \end{split}
        \end{equation}
        Define $\Omega_n$ as
        $$\Omega_n :=\left  \{ \b x\in [-1,1]^d: \Big| g(\b x)-h (\b x) \Big| \leq e^{-\sqrt{n}} \right \},$$
        and $\Omega_n^c = [-1,1]^d - \Omega_n$.
        Then \eqref{45_rn} implies
        \begin{align}\label{45_ap1}
            \begin{aligned}
                & \left\|  \chi_{\Omega}(\b x) - \tilde \chi_{\Omega}(\b x) \right \|_{w,p,\Omega_n^c}^p \\
                &:=\left\|  \chi_{\Omega}(\b x) - R_n  \circ \left( g-h \right)(\b x) \right \|_{w,p,\Omega_n^c}^p \\
                &\leq \int_{[-1,1]^d\backslash \Omega_n}\left| (\chi_{(0,\infty)}-R_n)  \circ \left( g-h \right)(\b x) \right|^pw(\b x) d\b x \\
                &\leq 2^pC_we^{-p\sqrt{n}}  \\
                &:= C_{p,w} e^{-p\sqrt{n}},
            \end{aligned}
        \end{align}
        where the constant $C_w $ only depends on $w$ and constant $C_{p,w}$ only depends on $p$, $w$. In the following, we denote $\tilde \chi_{\Omega}(\b x) = R_n  \circ \left( g-h \right)(\b x)$.

        Let functions $\tilde f_1(\b x)$ and $\tilde f_2(\b x)$ be SignReLU neural networks
        that satisfy $\|\tilde f_1 - f_1\|_{p,w}\leq \varepsilon$ and $\|\tilde f_2 - f_2 \|_{p,w} \leq \varepsilon$ (\thmref{sobolev}). Denote the product gate $\psi(x,y)=xy$ the SignReLU neural network in \lemref{times}.

        Then we have
        \begin{align}\label{45_ap2}
            \begin{aligned}
                &\left\|f_1 + f_2 \chi_{\Omega}  -\Big(\tilde f_1 + \psi\big( \tilde f_2 , \tilde \chi_{\Omega}\big) \Big)\right\|_{p,w,\Omega_n^c} \\
                &\leq \|f_1 -\tilde f_1\|_{p,w,\Omega_n^c} + \|f_2 \chi_{\Omega} - f_2 \tilde   \chi_{\Omega} \|_{p,w,\Omega_n^c} + \| f_2 \tilde \chi_{\Omega} -\tilde f_2 \tilde \chi_{\Omega}\|_{p,w,\Omega_n^c} \\
                &\leq \|f_1 -\tilde f_1\|_{p,w,\Omega_n^c}  + \| \chi_{\Omega} - \tilde \chi_{\Omega}(\b x)\|_{p,w,\Omega_n^c} + \|f_2 - \tilde f_2  \|_{p,w,\Omega_n^c} \\
                &\leq 2\varepsilon + C_{p,w} e^{-\sqrt{n}},
            \end{aligned}
        \end{align}
        where in the second step we use the property \eqref{45_rn} and in the third step we use the bound of $f_2$ on $[-1,1]^d$ and \eqref{45_ap1}.

        Denote $\tilde f:= \tilde f_1 + \psi\big( \tilde f_2 , \tilde \chi_{\Omega}\big) = \tilde f_1 + \psi\big( \tilde f_2 , R_n \circ (g-h) \big)$ and choose $n = \left(\ln (\varepsilon^{-1})\right)^2$. Then by \lemref{ra}, $R_n$ can be produced by a SignReLU neural network of $\L = O\left(\big(\ln (\varepsilon^{-1})\big)^2\right)$, $\W = O\left(1\right)$ and $\N = O\left(\big(\ln (\varepsilon^{-1})\big)^2\right)$. Combining \lemref{compos}, \lemref{parallel} and \lemref{concat}, the function $\tilde f$ is able to be implemented by a SignReLU neural network  with $\L = O\left( \varepsilon^{-\f{1}{r}} \right)$, $\W = O\left( \varepsilon^{-\f{d}{r}} \right)$ and $\N = O\left( \varepsilon^{-\f{d}{r}} \right)$ and satisfies $\|f-\tilde f\|_{p,w,\Omega_n^c}\leq 3\varepsilon$.

        Notice that \eqref{45_rn} guarantees the statement for functions from $W^r_\infty$ by using a similar step as \eqref{45_ap1} and \eqref{45_ap2}.

        Combining \thmref{koborov} with the above constructions, we can show the estimation of approximating piecewise Korobov functions similarly.

	\end{proof}
	

	\section{Proof of \thmref{BVfunctions}}

	\begin{proof}[{\bf Proof of \thmref{BVfunctions}}]
            We call a rational function $R(x) = p(x)/q(x)$ a type $(m,n)$ rational function if the degree of polynomials $p(x)$ and $q(x)$ is $m$ and $n$, respectively.

		Set $n\geq r$. Notice that Theorem 7.2 [Chapter 7, \cite{lorentz1996constructive}] shows that given $f \in V^r [0,1]$, there exists a rational function $R$ of type $(n,n)$, such that
        \begin{align}\label{bv_1}
            \begin{aligned}
                \|f-R\|_{L_\infty([0,1])} \leq \f{C_r}{n^{r+1}},
            \end{aligned}
        \end{align}
        for some constant $C_r$ depending only on $r$.

        Let $f \in \mathcal{V}^r_d$ with expression $f(\b x) = \prod_{i=1}^{d}f(x_i)$. Denote $R_i$ the rational function that approximates $f_i$ with error \eqref{bv_1}, $i=1,\dots,d$. Then we can get
		\begin{align}\label{bv_2}
            \begin{aligned}
                & \left|f(\b x) - \prod_{i=1}^{d}R_i(x_i)\right| \\
			&= \left|\sum_{i=1}^{d}\prod_{j=1}^{i-1}f_j(x_j)\left(f_i(x_i) - R_i(x_i) \right) \prod_{k=i+1}^{d}R_k(x_k)\right| \\
			&\leq \sum_{i=1}^{d}\prod_{j=1}^{i-1} \|f_j\|_{L_\infty([0,1])} \left\|f_i - R_i\right\|_{L_\infty([0,1])} \prod_{k=i+1}^{d} \left(\|f_k\|_{L_\infty([0,1])} + \|f_k - R_k\|_{L_\infty([0,1])}  \right) \\
			&\leq \f{C_{r,d}}{n^{r+1}},
            \end{aligned}
		\end{align}
		where $C_{r,d}$ is a constant and
            \begin{align}
                \begin{aligned}
                    \prod_{j=1}^{i-1}f_j(x_j) : = 1, \quad i=1, \\
                     \prod_{k=i+1}^{d}R_k(x_k) : = 1, \quad i=d.
                \end{aligned}
            \end{align}

		By \lemref{mono}, we denote the SignReLU neural network $\phi(\b x) = \prod_{i=1}^d x_i$. Define $\Phi(\b x):= \phi\left(R(x_1),\dots,R_d(x_d) \right)$ and let $n= (C_{r,d}/\varepsilon)^{1/(r+1)}$. Combining \lemref{compos}, \lemref{concat}, \lemref{ra} and \lemref{mono} with \eqref{bv_2}, we can see $\Phi$ is able to be represented by a SignReLU neural network with $\L = O\left( \varepsilon^{-\f{1}{r+1}} \right)$, $\W = O(1)$ and $\N = O( \varepsilon^{-\f{1}{r+1}} )$.
		

	\end{proof}
	
	
	\section{Proof of \thmref{continuous1}}

	\begin{proof}[{\bf Proof of \thmref{continuous1}}]

        Theorem 3.1 \cite{schultz1969multivariate} shows that for any $f \in C\left([0,1]^d\right)$, there exists a multivariate polynomial $P(\b x) = \sum_{\|\b n \|_\infty \leq Nd} c_{\b n} \b x^{\b n}$ such that $\|f - P\|_{L_\infty([0,1])} \leq \f{5}{4}\sum_{i=1}^{d} w^i_f(1/N)$. A similar proof as in the proof of \thmref{sobolev} verifies the statement.

	\end{proof}

	\end{appendix}





%
%
%

\end{document}